\PassOptionsToPackage{dvipsnames}{xcolor}
\AtBeginDocument{\RequirePackage{xcolor}}
\documentclass[twoside]{article}

\usepackage[accepted]{aistats2024}

\usepackage{hyperref}
\usepackage{url}
\usepackage{graphicx} 
\usepackage{hyperref}
\usepackage{amsmath}
\usepackage{amsthm}
\usepackage{amssymb}
\usepackage{booktabs}
\usepackage{pdflscape}
\usepackage[capitalise]{cleveref}
\usepackage{colortbl}
\usepackage{multirow}
\usepackage{siunitx}
\usepackage{algorithm}
\usepackage{algpseudocode}
\usepackage[title]{appendix}

\hypersetup{
  colorlinks,
  linkcolor={blue!70!black},
  citecolor={blue!70!black},
  urlcolor={blue!70!black}
}

\usepackage{amsmath,amsfonts,bm}

\def\1{\bm{1}}

\DeclareMathAlphabet{\mathsfit}{\encodingdefault}{\sfdefault}{m}{sl}
\SetMathAlphabet{\mathsfit}{bold}{\encodingdefault}{\sfdefault}{bx}{n}

\newcommand{\E}{\mathbb{E}}

\newcommand{\R}{\mathbb{R}}

\DeclareMathOperator*{\argmax}{arg\,max}
\DeclareMathOperator*{\argmin}{arg\,min}

\renewcommand{\R}{\mathbb{R}}
\renewcommand{\E}{\mathop{\mathbb{E}}}
\newcommand{\D}{\mathcal{D}}

\renewcommand{\L}{\mathcal{L}}

\newtheorem{theorem}{Theorem}

\newtheorem{lemma}[theorem]{Lemma}
\theoremstyle{definition}
\newtheorem{definition}[theorem]{Definition}

\usepackage[round]{natbib}
\renewcommand{\bibname}{References}

\begin{document}

%
\runningtitle{PreLoad}

%
\runningauthor{Rashid, Hacker, Zhang, Kristiadi, Poupart}

\twocolumn[

\aistatstitle{Preventing Arbitrarily High Confidence on Far-Away Data in Point-Estimated Discriminative Neural Networks}

\vspace{-1em}

\aistatsauthor{ Ahmad Rashid$^{1,4}$  \And Serena Hacker$^{2}$ \And  Guojun Zhang$^{3}$ \AND Agustinus Kristiadi$^{4}$ \And Pascal Poupart$^{1,4}$ }

\vspace{1em}

\aistatsaddress{University of Waterloo$^{1}$ \And University of Toronto$^{2}$ \And  Huawei Noah's Ark Lab$^{3}$ \And Vector Institute$^{4}$ }]



\vspace{1em}

\begin{abstract}
    Discriminatively trained, deterministic neural networks are the \emph{de facto} choice for classification problems.
    However, even though they achieve state-of-the-art results on in-domain test sets, they tend to be overconfident on out-of-distribution (OOD) data.
    For instance, ReLU networks---a popular class of neural network architectures---have been shown to almost always yield high confidence predictions when the test data are far away from the training set, even when they are trained with OOD data.
    We overcome this problem by adding a term to the output of the neural network that corresponds to the logit of an extra class, that we design to dominate the logits of the original classes as we move away from the training data.
    This technique \emph{provably} prevents arbitrarily high confidence on far-away test data while maintaining a simple discriminative point-estimate training.
    Evaluation on various benchmarks demonstrates strong performance against competitive baselines on both far-away and realistic OOD data.
\end{abstract}


\section{INTRODUCTION}
\label{sec:intro}
\begin{figure*}[t]
  \centering
  \includegraphics[width=\linewidth]{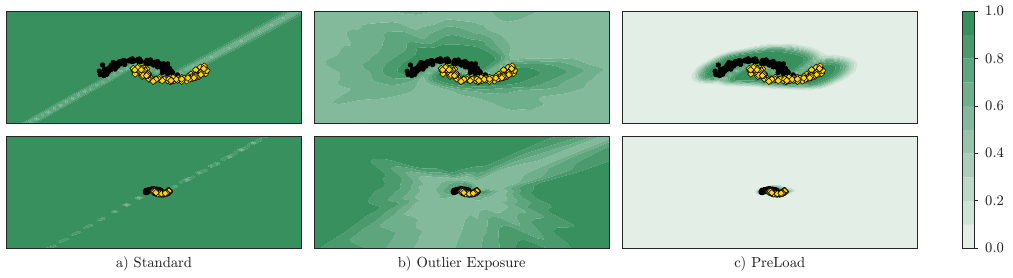}

  \caption{An illustrative example of the confidence of different methods trained on a synthetic binary classification dataset. The shades of green display the confidence of each algorithm with a darker shade signifying a higher confidence. The bottom row gives a zoomed-out view.}
  \label{fig:one}
\end{figure*}

Machine learning has made substantial progress over the last decade with the help of a strong deep learning toolkit, larger data sets, better optimization algorithms, faster and cheaper computation, and a vibrant research community. As machine learning systems continue to be deployed in safety-critical applications, important questions around their robustness and uncertainty quantification continue to be asked. A common expectation in uncertainty quantification is to assign high confidence to test cases close to the training data and low confidence to test cases that are out-of-distribution (OOD).

Recent advances in machine learning are in part due to deep neural networks (DNNs), which are powerful function approximators.
However, DNN classifiers tend to be overconfident for both in-domain examples~\citep{guo2017calibration} and data that is far away from the training examples~\citep{nguyen2015deep}.
\citet{hein2019relu} showed that the ubiquitous ReLU Networks almost always exhibit high confidence on samples that are far away from the training data. 

A number of methods have been proposed to deal with the overconfidence issue in DNNs.
Calibration methods attempt to solve overconfidence of neural network classifiers by various methods including smoothing the softmax distribution~\citep{guo2017calibration,gupta2020calibration,kull2019beyond}, regularization~\citep{muller2019does,thulasidasan2019mixup} and adding additional constraints to the loss function~\citep{kumar2018trainable,lin2017focal}. These methods, however, do not resolve overconfidence issues around OOD data~\citep{NEURIPS2021_8420d359}.
Other methods, both Bayesian~\citep{blundell2015weight,gal2016dropout,kristiadi2020being} and non-Bayesian~\citep{lakshminarayanan2017simple,mukhoti2023deep} have improved OOD detection while training only with the in-domain data distribution.

State-of-the-art methods for OOD detection are typically trained with additional OOD training data with the goal for the classifier to output either high ``None'' class probability~\citep{zhang2017universum,kristiadi2022being} or uniform confidence~\citep{hendrycks2018deep}, in the presence of OOD samples.
\citet{hein2019relu} showed that there is no guarantee that OOD data would be predicted as the ``None'' class.
Moreover, we will demonstrate that these methods still exhibit high confidence when the test points are far away from the data. 

One way of overcoming the problem of arbitrarily high confidence on far-away data is to incorporate generative modeling, either as a \emph{posthoc} method~\citep{lee2018simple,mukhoti2023deep} or as a prior on the data~\citep{meinke2020towards}, into a neural network.
The former assumes that the neural network embedding can be approximated with a Gaussian distribution.
However, on realistic, OOD data, we will demonstrate that these methods are not competitive with the state-of-the-art. The latter assumes a generative model over the data which is a harder problem than the underlying discriminative modeling.

Finally, while Bayesian neural networks \citep{louizos2017multiplicative,kristiadi2022posterior} have also been used to overcome this issue, they are not guaranteed to obtain the optimal confidences on far-away OOD test data \citep{kristiadi2020being}.
While a more sophisticated remedy exists for this \citep{kristiadi2021infinite}, they are specifically constructed to only fix the far-away high confidence, and their detection performance on `nearby' OOD data are more of an afterthought.

In this work, we present our method, called \textbf{Pr}oducing Larg\textbf{e}r \textbf{Lo}gits \textbf{a}way from \textbf{D}ata, or \textbf{PreLoad}, which fulfills the following desiderata: (i) it must maintain the simplicity of the standard discriminative training procedure for DNNs (unlike generative- and Bayesian-based methods), (ii) it must provably be less confident on inputs far away from the training data, and (iii) it must perform well on realistic OOD examples (e.g.\ CIFAR-10 vs.\ CIFAR-100).

We accomplish this by training an extra class, such that under an OOD input, this extra logit is larger than the logits of the other classes as we move farther away from the training data.
This construction provably helps PreLoad almost always predict far-away data as OOD.
Furthermore, the extra class is trained on an auxiliary, OOD dataset, which helps it detect realistic, nearby OOD examples well.

Figure~\ref{fig:one} illustrates the confidence level of PreLoad as we move away from the training data, compared to a standard-trained neural network and a discriminative OOD training approach called Outlier Exposure \citep[OE,][]{hendrycks2018deep}.
Standard neural networks with a softmax output layer exhibit high confidence as we move away from the decision boundary.
OE's confidence initially decreases away from the data, but it becomes high far away as we zoom out. In contrast, PreLoad is confident when close to the data and uncertain when away from it.

\section{PRELIMINARIES}
\label{sec:prelim}

We define a neural network as a function $f: \R^n \times \R^p \to \R^k$ with $(x,\theta) \mapsto f_\theta (x)$, where $\R^n$ is the input space, $\R^k$ the output space, and $\R^p$ the parameter space.
Let $\D:=\{(x_i, y_i)\}_{i=1}^m$ be a training dataset. The standard way of training a neural network is by finding optimal parameters $\theta^*$ such that $\theta^*=\argmin_{\theta}\sum_{i=1}^{m}\ell(f_\theta(x_i),y_i)$ for some loss function $\ell$ such as the cross-entropy loss for classification.

One of the most widely used neural network architectures is a ReLU network. We use the term ReLU networks for feedforward neural networks with piecewise affine activation functions, such as the ReLU or leaky ReLU activation, and a linear output layer. ReLU networks can be written as continuous piecewise affine functions \citep{arora2018understanding,hein2019relu}.

\begin{definition}
    A function $f: \R^n \to \R $ is called piecewise affine if there exists a set of polytopes $\{ Q_r \}_{r=1}^M$ such that their union is $\R^n$ and $f$ is affine in each polytope \citep{arora2018understanding,hein2019relu}.
\end{definition}
Piecewise affine functions include networks with fully connected layers, convolution layers, residual layers, skip connections, average pooling, and max pooling. We will rely on the neural network being a continuous piecewise affine function to prove that our algorithm prevents arbitrarily high confidence on far-away data.





Consider a classification problem where $x$ is the input and $y\in \{1,\cdots,k\}$ denotes the target class. A neural network with a linear output layer in conjunction with the softmax link function can be used to compute the probability $P(y|x)$.  More precisely, consider the following decomposition of the neural network $f_\theta(x)=WG_{\psi}(x)$ where $W \in \R^{k \times d}$ is the weight matrix for the last layer, $G_\psi(x) \in \R^d$ is the neural network embedding and $\theta=\{\psi,W\}$. Each row of $f$ corresponds to the logit $z_c(x)$ of class \(c\): 
\begin{equation}
\label{eq:logit}
    z_c(x) = w_c^\top G_{\psi}(x) + b_c.
\end{equation}
Then the last layer computes class probabilities via a softmax such that:
\begin{equation}
P(y=c|x)=\frac{\exp(w^\top_c G_{\psi}(x) + b_c)}{\sum_{c'=1}^k \exp(w^\top_{c'} G_{\psi}(x) + b_{c'})}
\label{eq:softmax}
\end{equation}
where $w_c \in \R^d$ and $b_c \in \R$ are the parameters of the last layer associated with class $c\in\{1,\cdots,k\}$.

Generally, learning $P(y|x)$ is referred to as discriminative modeling. Generative models, such as GANs~\citep{goodfellow2014generative} and VAEs~\citep{kingma2013auto} learn the distribution of the data $P(x)$.
Meanwhile, class-conditional generative models \citep{mukhoti2023deep} learn $P(x|y)$.

\subsection{Arbitrarily High Confidence on Far-Away Data}

Arbitrarily high confidence on far-away data i.e. data which is far away from the training set \citep{hein2019relu}, can be formalized as observing that the probability of some class approaches $1$ in the limit of moving infinitely far from the training data.


\begin{definition}
A model exhibits far-away arbitrarily high confidence if there exists $x \in \R^n$ and $c \in \{1,\cdots,k\}$ such that
\begin{equation}
\lim_{t\rightarrow\infty} P(y=c|t x) = 1.
\end{equation}
\end{definition}


\cite{hein2019relu} showed that piecewise affine networks (including ReLU networks) with a linear last layer almost always exhibit arbitrarily high confidence far away from the training data.

\section{METHODOLOGY}
\label{sec:method}
\begin{figure*}[t]
    \centering
    \includegraphics[width=\linewidth]{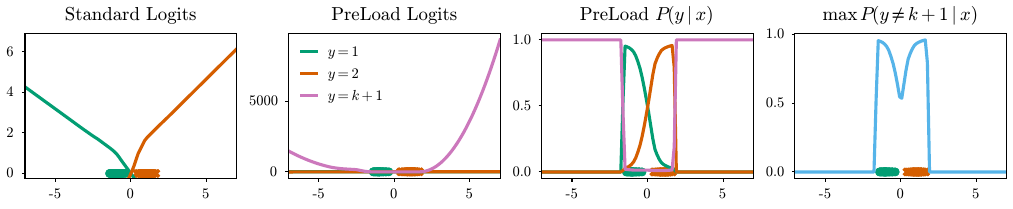}
    \caption{Effect of training with an OOD class with our method on a $1$-D binary classification problem. Standard logits keep on growing when away from the data. We implement the OOD class such that the logits grow much faster for the OOD class compared to the in-domain class. This `fixes' the probabilities and the confidence away from the dataset. Note that the range of y values is larger on the second plot.
    }
    \label{fig:method}
\end{figure*}

Consider a neural network, $f_{\psi,W}$ , trained on a $k$-class classification problem such that the logit $z_c$ is defined according to \eqref{eq:logit} and $P(y|x)$ is computed according to \eqref{eq:softmax}. Arbitrarily high confidence arises when the logit of one class becomes infinitely higher than the logits of the other classes:
\begin{lemma}
    \label{lemma:logits}
    Let \(P(y | x)\) be a classifier defined in \eqref{eq:softmax} and let \(x \in \R^n\).
    If the classifier exhibits arbitrarily high confidence on far-away inputs (i.e., $\lim_{t\rightarrow\infty} P(y|t x) = 1$), then there must exist $c \in \{ 1, \dots, k \}$ such that $\lim_{t\rightarrow\infty}z_c(tx) - z_{c'}(tx) = +\infty$ for all $c'\neq c$.
\end{lemma}

\begin{proof}
    From \eqref{eq:softmax}, if $\lim_{t\rightarrow\infty} P(y=c|tx) = 1$, then we have:
    \begin{align*}
    \lim_{t\rightarrow\infty} \frac{\exp(z_c(tx))}{\sum_{c'} \exp(z_{c'}(tx))} = 1.
    \end{align*}
    Therefore,
    \begin{align*}
    & \lim_{t\rightarrow\infty} \frac{\exp(z_c(tx))}{\sum_{c'} \exp(z_{c'}(tx))} = 1 \\
    \implies & \lim_{t\rightarrow\infty} \frac{\sum_{c'} \exp(z_{c'}(tx))}{\exp(z_c(tx))} = 1 \\
    \implies & \lim_{t\rightarrow\infty} 1+\sum_{c'\neq c}\frac{\exp(z_{c'}(tx))}{\exp(z_c(tx))} = 1 \\
    \implies & \lim_{t\rightarrow\infty} \sum_{c'\neq c}\exp(z_{c'}(tx)-z_c(tx)) = 0 \\
    \implies & \lim_{t\rightarrow\infty} \exp(z_{c'}(tx)-z_c(tx)) = 0 \quad \forall c'\neq c .
    \end{align*}
    Thus, we can conclude that \(\lim_{t\rightarrow\infty} z_c(tx) - z_{c'}(tx) = +\infty\) for all \(c' \neq c\).
\end{proof}

An immediate consequence of the above lemma is that networks with normalization such as layernorm do not  suffer from far-away arbitrarily high confidence since the layers that follow layernorm (including the logits) will remain bounded.  Note that networks with batchnorm may still exhibit far-away arbitrarily high confidence since batchnorm ensures that the logits of the training set are bounded, but not necessarily the logits of the test set, which may include OOD data that could be arbitrarily far.

Our solution consists of creating an additional, $(k+1)$-st class such that the confidence in the original classes vanishes far away from the training data while the confidence in the $(k+1)$-st class becomes arbitrarily high.
Note that this is desirable since the extra class represents the OOD class.
Based on Lemma~\ref{lemma:logits}, we achieve this by making sure that the corresponding logit $z_{k+1}$ is infinitely higher than the logits $z_{c\in\{1,...,k\}}$ of the other classes far away from the training data.  More precisely, let
\begin{equation}
    \label{eq:none-logit}
    z_{k+1}(x)=w_{k+1}^\top G_\psi(x)^2 + b_{k+1}
\end{equation}
where the weights $w_{k+1} \in \R^d_{> 0}$ are restricted within the positive orthant and $G_\psi(x)^2$ is the component-wise square of the network embedding \(G_\psi(x)\).

Then, given these logits, classification is performed as follows:
\begin{equation}
    \label{eq:c}
    P(y=c|x) = \frac{1}{A(x)} \exp(w_c^\top G_\psi(x) + b_c)
\end{equation}
for \(c \in \{ 1, \dots, k \}\), and
\begin{equation}
    \label{eq:none}
    P(y=k+1|x) = \frac{1}{A(x)} \exp(w_{k+1}^\top G_\psi(x)^2 + b_{k+1}) ,
\end{equation}
where
\begin{equation}
    \begin{aligned}
        A(x) := \sum_{c'=1}^k &\exp(w_{c'}^\top G_\psi(x) +  b_c) \\
            &+ \exp(w_{k+1}^\top G_\psi(x)^2 + b_{k+1})
    \end{aligned}
\end{equation}
is the softmax's denominator.

Intuitively, as we move away from the training data the magnitude of $G_\psi(x)$ may also increase, which may result in some class $c$ dominating with arbitrarily high confidence \citep{hein2019relu}.
However, by using $G_\psi(x)^2$ in the logit of the $(k+1)$-st class we make sure that it grows faster than other logits and therefore, eventually dominates. In Theorem~\ref{thm:no-high-confidence}, we prove that any classification network augmented with such a construction never exhibits far-away arbitrarily high confidence in classes $\{1,...,k\}$. Note that the theorem holds if the exponent of the term $G_\psi(x)$ in \eqref{eq:none-logit} is replaced with another even integer greater than 2. 

\begin{theorem}
    \label{thm:no-high-confidence}
    Let $G_\psi$ be any neural network embedding used for classification according to \eqref{eq:c} and \eqref{eq:none}. Let $w_c$ and $b_c$ be finite weights and biases in the penultimate classification layer for each class $c$.
    Let $tx_* \in \R^n$ be a test input with magnitude regulated by $t$.
    Then $\lim_{t\rightarrow\infty}P(y=c|tx_*) < 1$ for all $c \neq k+1$.
\end{theorem}

\begin{proof}
    Based on Lemma~\ref{lemma:logits}, arbitrarily high confidence (i.e., $\lim_{t\rightarrow\infty} P(y=c|tx) = 1$) arises when there is a $c$ such that $\lim_{t\rightarrow\infty} z_c(tx) - z_{c'}(tx) = \infty $ $\forall c'\neq c$.
    We prove by contradiction that this cannot happen once we introduce the extra class with its logit as defined in \eqref{eq:none-logit}.  Consider two cases:
    \begin{enumerate}
        \item Suppose that there exists a class $c\neq k+1$ such that $\lim_{t\rightarrow\infty}z_c(tx)=\infty$ and for all $c'\neq c$, $\lim_{t\rightarrow\infty} z_c(tx) - z_{c'}(tx)=\infty$.  Since $z_c(tx) = w_c^\top G_\psi(x) + b_c$, and the weights and biases are finite, then $\lim_{t\rightarrow\infty} z_c(tx)=\infty$ implies that $\lim_{t\rightarrow\infty} \Vert G_\psi(tx) \Vert =\infty$.  Since $z_{k+1}(tx)=w_{k+1}^\top G_\psi(x)^2 + b_{k+1}$ where $w_{k+1}\in \R^d_{> 0}$, i.e.\ each component of $w_{k+1}$ is positive, and $G_\psi(x)^2$ is always component-wise positive, then $\lim_{t\rightarrow\infty} z_{k+1}(tx)=\infty$ and $\lim_{t\rightarrow\infty} z_{k+1}(tx)> \lim_{t\rightarrow\infty} z_c(tx)$, which contradicts the assumption that $\lim_{t\rightarrow\infty} z_c(tx) - z_{k+1}(tx)=\infty$.
        \item Suppose that there exists a class $c\neq k+1$ such that $\lim_{t\rightarrow\infty} z_c(tx)<\infty$ and for all \(c'\neq c\), $\lim_{t\rightarrow\infty} z_c(tx) - z_{c'}(tx)=\infty$. Since $z_{k+1}=w_{k+1}^\top G_\psi(x)^2+b_{k+1}$, $w_{k+1}\in \R^d_{>0}$, $G_\psi(x)^2>0$ and $b_{k+1}<\infty$, then $\lim_{t\rightarrow\infty} z_{k+1}(tx) >-\infty$, which contradicts the assumption that $\lim_{t\rightarrow\infty} z_c(tx)-z_{k+1}(tx)=\infty$.
    \end{enumerate}
    Altogether, they imply the desired result.
\end{proof}

The above theorem guarantees that arbitrarily high confidence will not occur for any neural network with an extra class that we propose.  In addition, we show a stronger result in \cref{thm:vanishing-confidence} for ReLU classification networks.  As we move far away from the training data, we show that the confidence in the original classes (i.e., $c\in\{1,...,k\}$) will be dominated by the extra class. To prove this, we first recall an important lemma from \cite{hein2019relu} about ReLU networks.

\begin{lemma}[\citeauthor{hein2019relu}, \citeyear{hein2019relu}] \label{lemma:hein}
   Let $\{Q_r\}_{r=1}^R$ with \(\R^n = \cup_{r = 1}^R Q_r\) be a set of linear regions associated with a ReLU network $G_\psi:\R^n \to \R^d$. For any $x \in \R^n$ there exists an $\alpha \in \R_{>0}$ and $r \in \{1,\cdots,R \}$ such that for all \(t \geq \alpha\), we have $t x \in Q_r$. \qed
\end{lemma}

This lemma tells us that as we move far away from the data region via scaling an input \(x \in \R^n\) with a nonnegative scalar, at some point we can represent the ReLU network with just an affine function.
It follows that in this case, increasing the scaling factor makes the magnitude of the network's output larger.
We will use this fact in our main theoretical result.

\begin{theorem} \label{thm:vanishing-confidence}
    Let $G_\psi(x)$ be a ReLU network embedding used for classification according to \eqref{eq:c} and \eqref{eq:none} with a piecewise affine representation \(G_\psi \vert_{Q_r} (x) = V_\psi^rx + a_{\psi}^r\) on the linear regions \(\{ Q_r \}_{r=1}^R\),  where $V_{\psi}^r \in \mathbb{R}^{d \times n}$ and $a_{\psi}^r \in \mathbb{R}^d$.
    Suppose \(V_\psi^rx\) does not contain identical rows for all \(r = 1, \dots, R\).
    Then for almost any input \(x_* \in \R^n\), we have \(\lim_{t \to \infty} \argmax_{c=1, \dots, k+1} P(y = c | t x_*) = k+1\).
\end{theorem}

\begin{proof}
    First, since the coefficients $w_{k+1} \in \R^d_{>0}$ are constrained to be component-wise positive, the logit $w_{k+1}^{\top} G_\psi(x_*)^2$ of the additional \((k+1)\)-st class is always positive.
    Second, by \cref{lemma:hein}, there exists \(\alpha > 0\) s.t.\ for all \(t \geq \alpha\), the ReLU network is represented by a single affine function $G_\psi(t x_*) = V_\psi t x_* + a_\psi$.
    Therefore, as $t \to \infty$, the norm $\Vert G_\psi(t x_*)^2 \Vert$ of $G_\psi(t x_*)^2 = (t V_\psi x_* + a_\psi)^2$ also tends to infinity (recall that we use the notation \((\cdot)^2\) on vectors as component-wise square).

    Now, notice that we can write $P(y=k+1 | x=t x_*)$ as:
    %
        \begin{align}
    \frac{e^{w_{{k+1}}^{\top} G_\psi(t x_*)^2 +b_{k+1}} }{\sum_{c'=1}^ke^{w_{c'}^\top G_\psi(t x_*)+b_{c'}}+ e^{w_{{k+1}}^{\top} G_\psi(t x_*)^2 +b_{k+1}} } 
        \end{align}
which is equal to:
\begin{align}
\frac{1}{1+\sum_{c'=1}^k e^{w_{c'}^\top G_\psi(t x_*) + b_{c'} -w_{k+1}^\top G_\psi(t x_*)^2 - b_{k+1} } }
\end{align}

    %
    Recall that \(\lim_{t \to \infty} \Vert G_\psi(t x_*) \Vert \to \infty\).
    Moreover, \(\Vert G_\psi(t x_*)^2 \Vert\) grows even faster.
    So, as \(t \to \infty\)  we can see from the expression above that \(P(y=k+1 | x=t x_*) = \frac{1}{1+k \exp(-\infty)} = 1\).
    This immediately implies that the class \(k+1\) achieves the maximum softmax probability since probability vectors sum to one.
    Moreover, the in-distribution classes \(\{ 1, \dots, k \}\) have the probabilty zero.
\end{proof}

Figure~\ref{fig:method} shows the effect of our method on the prediction and confidence of a neural network classifier on a one-dimensional binary classification toy dataset. We can observe that the standard logits keep on increasing as we move away from the data. Therein we train an extra class using uniform noise and observe that as we move away from the training data the logits of the extra class dominate. This `fixes' the neural network prediction and confidence away from the training data as far-away inputs are predicted as the extra class.

In order to train the extra class we rely on an auxiliary OOD dataset like previous methods~\citep{hendrycks2018deep,meinke2020towards}. Such methods tend to demonstrate strong performance on OOD detection on standard benchmarks. Our overall training objective is as follows:
\begin{align*}
    \label{eq:loss}
    \begin{split}
     \mathcal{R}(f_{\psi,W}(x)) & ={\E}_{(x,y)\sim \D_\text{in}} \L_\text{CE}(f_{\psi,w_{k \in \{1,\cdots, K\}}}(x),y) \\ & + \lambda {\E}_{(x')\sim \D_\text{ood}} \L_\text{CE}(f_{\psi,w_{k+1}}(x'),k+1).
    \end{split}
\end{align*}

Here $\L_{CE}$ is the cross-entropy loss and $\lambda$ controls the relative weight between the loss on in-domain and OOD training inputs. Algorithm~\ref{alg:method} shows the training procedure for PreLoad. Note that we have the option of either training our method from scratch or fine-tuning after a neural network has been trained on in-domain data, similar to \citet{hendrycks2018deep}.

\begin{algorithm}[t]
    \small

    \caption{PreLoad Algorithm}
    \label{alg:method}

    \begin{algorithmic}[1]
        \Require
            \Statex Training Set $\D_\text{in}:=\{(x_i \in \R^n, y_i \in \{1,
            \cdots, k\}) \}$
            \Statex OOD Training Set $\D_\text{ood}:=\{(x'_i \in \R^n)\}$
            \Statex Neural network $f_\theta$ with $\theta=\{\psi,W\}$, number of iterations \(T\), learning rate $\eta$
        \vspace{1em}

        \For {$i\leftarrow 1$ \textbf{to} $T$}
            \State Sample a mini-batch S from $\D_{in}: S = \{x_i,y_i\}_{i=1}^m$
            \State Sample a mini-batch S' from $\D_{ood}: S' = \{x_{i}'\}_{i=1}^m$
            \State Compute the objective function $\mathcal{R}$ such that
            \State $ \mathcal{R}(f_{\psi,W}(x)) =\E_{S} \L_{CE}(f_{\psi,w_{k \in \{1,\cdots, K\}}}(x_i),y_i)$  \\ \hspace{2.5cm}
            $+ \lambda {\E}_{S'} \L_{CE}(f_{\psi,w_{k+1}}(x_i'),k+1)$ 
            \State Update the parameters $\theta_{t+1} = \theta_{t} - \eta \nabla_\theta \mathcal{R}(f_{\psi,W}(x))$
        \EndFor

        \State \textbf{Predict} OOD: $\mathrm{is\_ood} = (\argmax_c P(y=c|x_*) == k+1)$ for test sample $x_*$
    \end{algorithmic}
\end{algorithm}

\section{RELATED WORKS}
\label{sec:related}

\paragraph{Gaussian Assumption.} Some recent works on OOD detection have assumed that the embedding, $G_\psi(x)$, produced by the penultimate layer of a neural network is Gaussian, and have built algorithms based on this. \cite{lee2018simple} propose fitting class conditional Gaussians on $G(x)$ such that $p(G_{\theta}(x)|y=c) = \mathcal{N}(G_{\theta}(x|{\mu}_c, \Sigma)$ where $\Sigma$ is a diagonal covariance matrix. The mean and covariance are empirically estimated from the  class-wise neural network embeddings of the training data. The method computes a confidence score, for a test sample $x_i$, based on the Mahalanobis distance from the class-conditional Gaussians. \cite{mukhoti2023deep} go further by fitting a Gaussian Mixture Model (GMM) on $G_\psi(x)$. Thereafter, they use the GMM density as the OOD score for a test sample $x_i$.

Even though these methods are deterministic and prevent arbitrarily high confidence on far-away data, they assume that $G_\psi(x)$ follows a Gaussian or mixture of Gaussian distribution. Moreover, they need to adjust a confidence threshold for each dataset and require an additional step beyond standard discriminative training to fit the Gaussian or mixture of Gaussians.

\textbf{OOD training.} \cite{zhang2017universum} presented the concept of a ``None'' class or an additional class for a supervised learning problem, which, is trained on unlabeled data for regularization of a DNN to improve generalization. \cite{kristiadi2022being} adapted this method to OOD detection such that they train an additional output of a neural network to predict a ``None'' class. The linear layer weights corresponding to the ``None'' class, $w_{k+1}$ are trained on an additional OOD data set which is carefully selected to remove any overlap with the training set. Even though using an OOD set may not be ideal, ~\cite{kristiadi2022being} demonstrated that these methods show state-of-the-art performance. \cite{hein2019relu} showed that theoretically, ``None'' class methods are prone to arbitrarily high confidence on far-away data.

Outlier Exposure \citep[OE,][]{hendrycks2018deep} also relies on OOD data, but, instead of learning an extra class, trains the class probabilities, $P(y|x)$ to output a uniform distribution when the data is OOD. Distributional-agnostic Outlier Exposure \citep[DOE,][]{wang2023outofdistribution}, is a variant of OE, which, uses model perturbation to generate ``worst-case'' OOD data and applies the OE algorithm on these data. We will demonstrate in the results that these methods also fail in the presence of far-away data. \cite{meinke2020towards} present an algorithm that models a joint probability distribution, $P(x,y)$ over both the in-distribution and OOD data. Using this, they jointly train a neural network that models the predictive distribution and two GMMs that model the generative distribution for in-domain and OOD data. Similar to OE, the neural network is trained to output uniform probabilities for OOD data. This algorithm has some provable guarantees on far-away OOD detection and reaches close to the performance of OE on standard benchmarks. Our algorithm, however, can achieve that without any generative modeling on either in-domain or OOD data.

An alternative to relying on the softmax for confidence is using an energy function. \cite{liu2020energy} propose a fine-tuning algorithm that combines an energy-based loss function with the standard cross-entropy loss. This additional loss uses two additional margin hyper-parameters, $\{m_{in},m_{ood}\}$, and penalizes in-domain samples which produce energy higher than $m_{in}$ and OOD samples which produce energy lower than $m_{ood}$. They only rely on discriminative training but do not prevent arbitrarily high confidence on far-away data, which, we will demonstrate in the results.





\vspace{-0.2em}
\section{EXPERIMENTS}
\vspace{-0.2em}
\label{sec:experiments}
\begin{table*}[t]
  \caption{
    OOD data detection using the FPR-$95$ metric when the OOD data is far away from the training data. We present the average result of $5$ runs with error bars. Lower numbers are better.
  }
  \label{tab:faraway-fpr95}

  \vspace{0.5em}

  \scriptsize

  \setlength{\aboverulesep}{0pt}
  \setlength{\belowrulesep}{0pt}
  \setlength{\extrarowheight}{.75ex}

    \begin{tabular*}{\linewidth}{l@{\extracolsep{\fill}} c c c c  >{\columncolor{black!20}}c || c c c >{\columncolor{black!20}}c }

    \toprule

    \multirow{2}{*}{\textbf{Dataset}} & \multicolumn{5}{c}{\textbf{Trained from Scratch}} & \multicolumn{4}{c}{\textbf{Finetuned}} \\

    \cmidrule(l){2-6} \cmidrule(l){7-10}

     & \textbf{Standard} & \textbf{DDU} & \textbf{NC} & \textbf{OE} & \textbf{PreLoad} & \textbf{OE-FT} & \textbf{DOE-FT} & \textbf{Energy-FT} & \textbf{PreLoad-FT} \\

    \midrule

    \textbf{MNIST} & & & & & & & & &  \\
    FarAway & 100.0$\pm$0.0 & \textbf{0.0$\pm$0.0} & \textbf{0.0$\pm$0.0} & 56.6$\pm$19.6 & \textbf{0.0$\pm$0.0} & 99.0$\pm$0.4 & 56.8$\pm$18.1 & 100.0$\pm$0.0 & \textbf{0.0$\pm$0.0} \\
    FarAway-RD & 99.9$\pm$0.0 & \textbf{0.0$\pm$0.0} & 99.9$\pm$0.1 & 99.8$\pm$0.0 & \textbf{0.0$\pm$0.0} & 99.5$\pm$0.1 & 99.8$\pm$0.1 & 100.0$\pm$0.0 & \textbf{0.0$\pm$0.0} \\

     \midrule

    \textbf{F-MNIST} & & & & & & & & &  \\
    FarAway & 100.0$\pm$0.0 & \textbf{0.0$\pm$0.0} & 53.5$\pm$22.5 & 100.0$\pm$0.0 & \textbf{0.0$\pm$0.0} & 100.0$\pm$0.0 & 99.6$\pm$0.4 & 38.4$\pm$8.9 & \textbf{0.0$\pm$0.0} \\
    FarAway-RD & 100.0$\pm$0.0 & \textbf{0.0$\pm$0.0} & 100.0$\pm$0.0 & 100.0$\pm$0.0 & \textbf{0.0$\pm$0.0} & 100.0$\pm$0.0 & 100.0$\pm$0.0 & 81.6$\pm$8.8 & \textbf{0.0$\pm$0.0} \\

    \midrule

    \textbf{SVHN} & & & & & & & & &  \\
            FarAway & 99.4$\pm$0.2 & \textbf{0.0$\pm$0.0} & 80.0$\pm$20.0 & 99.4$\pm$0.4 & \textbf{0.0$\pm$0.0} & 99.3$\pm$0.3 & 99.9$\pm$0.1 & 100.0$\pm$0.0 & \textbf{0.0$\pm$0.0} \\
    FarAway-RD & 99.8$\pm$0.1 & \textbf{0.0$\pm$0.0} & 80.0$\pm$20.0 & 85.4$\pm$7.6 & \textbf{0.0$\pm$0.0} & 93.1$\pm$2.5 & 99.3$\pm$0.6 & 100.0$\pm$0.0 & \textbf{0.0$\pm$0.0} \\

    \midrule

    \textbf{CIFAR-10} & & & & & & & & &  \\
    FarAway & 100.0$\pm$0.0 & \textbf{0.0$\pm$0.0} & 20.0$\pm$20.0 & 100.0$\pm$0.0 & \textbf{0.0$\pm$0.0} & 100.0$\pm$0.0 &100$\pm$0.0 & 100.0$\pm$0.0 & \textbf{0.0$\pm$0.0} \\      
    FarAway-RD & 99.7$\pm$0.2 & \textbf{0.0$\pm$0.0} & 40.0$\pm$24.5 & 100.0$\pm$0.0 & \textbf{0.0$\pm$0.0} & 99.5$\pm$0.3 & 99.6$\pm$0.4 & 100.0$\pm$0.0 & \textbf{0.0$\pm$0.0} \\

    \midrule

    \textbf{CIFAR-100} & & & & & & & & &  \\
    FarAway & 100.0$\pm$0.0 & \textbf{0.0$\pm$0.0} & 20.0$\pm$20.0 & 100.0$\pm$0.0 & \textbf{0.0$\pm$0.0} & 100.0$\pm$0.0 & 100.0$\pm$0.0 & 100.0$\pm$0.0 & \textbf{0.0$\pm$0.0} \\
    FarAway-RD & 100.0$\pm$0.0 & \textbf{0.0$\pm$0.0} & 20.0$\pm$20.0 & 100.0$\pm$0.0 & \textbf{0.0$\pm$0.0} & 100.0$\pm$0.0 & 100.0$\pm$0.0 & 100.0$\pm$0.0 & \textbf{0.0$\pm$0.0} \\

    \bottomrule

    \end{tabular*}
\end{table*}

\begin{table*}
  \caption{
    OOD detection results on image classification data reporting the FPR-$95$ metric. The results are averaged over $6$ OOD test sets and five runs for each instance. Lower numbers are better.
  }
  \label{tab:fpr_avg}

  \vspace{0.5em}

  \scriptsize
  \setlength{\aboverulesep}{0pt}
  \setlength{\belowrulesep}{0pt}
  \setlength{\extrarowheight}{.75ex}

  \begin{tabular*}{\linewidth}{l@{\extracolsep{\fill}} c c c c >{\columncolor{black!20}}c || c c c >{\columncolor{black!20}}c }

    \toprule

    \multirow{2}{*}{\textbf{Dataset}} & \multicolumn{5}{c}{\textbf{Trained from Scratch}} & \multicolumn{4}{c}{\textbf{Finetuned}} \\

    \cmidrule(l){2-6} \cmidrule(l){7-10}

     & \textbf{Standard} & \textbf{DDU} & \textbf{NC} & \textbf{OE} & \textbf{PreLoad} & \textbf{OE-FT} & \textbf{DOE-FT} & \textbf{Energy-FT} & \textbf{PreLoad-FT} \\

    \midrule

    {MNIST} & 10.9$\pm$2.3 & 47.7$\pm$6.9 & \textbf{3.3$\pm$1.2} & \textbf{5.5$\pm$1.9} & 6.6$\pm$2.0 & \textbf{4.7$\pm$1.7} & \textbf{4.3$\pm$1.6} & \textbf{8.4$\pm$2.4} & \textbf{8.4$\pm$2.5} \\
    {F-MNIST} & 70.6$\pm$4.0 & 35.1$\pm$7.7 & \textbf{2.2$\pm$0.5} & 31.7$\pm$5.5 & \textbf{2.3$\pm$0.6} & 31.7$\pm$5.9 & 20.7$\pm$3.9 & \textbf{14.5$\pm$2.9} & \textbf{12.4$\pm$2.3} \\
    {SVHN} & 23.7$\pm$1.3 & 8.0$\pm$1.6 & \textbf{2.1$\pm$0.9} & \textbf{1.7$\pm$0.7} & \textbf{1.1$\pm$0.5} & \textbf{1.6$\pm$0.6} & \textbf{1.3$\pm$0.5} & 7.2$\pm$0.9 & \textbf{0.8$\pm$0.3} \\
    {CIFAR10} & 51.1$\pm$3.6 & 38.0$\pm$4.8 & \textbf{5.7$\pm$1.7} & 11.7$\pm$2.1 & \textbf{6.0$\pm$1.8} & 20.0$\pm$2.8 & \textbf{15.1$\pm$2.5} & \textbf{15.6$\pm$3.6}& \textbf{12.0$\pm$2.5} \\
    {CIFAR100} & 77.2$\pm$1.9 & 66.0$\pm$6.4 & \textbf{27.5$\pm$5.2} & 60.2$\pm$4.1 & \textbf{25.9$\pm$4.8} & 70.6$\pm$2.5 & 54.3$\pm$4.4 & 49.4$\pm$4.6 & \textbf{39.5$\pm$5.2} \\
    \bottomrule
  \end{tabular*}

\end{table*}

We evaluate our algorithm, PreLoad,\footnote{\url{https://github.com/serenahacker/PreLoad}} in three ways. First, we evaluate on synthetic far-away data to validate our theoretical results. Then, we evaluate on standard benchmarks which measure OOD detection performance on realistic data. Finally, we evaluate the calibration of our model under dataset shift.

\textbf{Datasets.} Our in-domain datasets include MNIST~\citep{lecun1998gradient}, Fashion MNIST (FMNIST)~\citep{xiao2017fashion}, SVHN~\citep{netzer2011reading}, CIFAR10 and CIFAR100~\citep{krizhevsky2009learning}. We train LeNet for MNIST and FMNIST and WideResNet-16-4~\citep{zagoruyko2016wide} for SVHN, CIFAR10 and CIFAR100. The OOD training set for the methods which rely on OOD training, including ours, is $300,000$ random images  as released by \cite{hendrycks2018deep} as $80$ million tiny images~\citep{torralba200880} is no longer available.

\textbf{Metrics.} Following convention, we define an in-domain sample as positive and an OOD sample as negative. The true positive rate (TPR) is $\text{TPR} = \frac{\text{TP}}{\text{TP} + \text{FN}}$ and the false positive rate (FPR) is $\text{FPR }= \frac{\text{FP}}{\text{FP} + \text{TN}}$, where TP, FN, FP and TN are true positive, false negative, false positive and true negative respectively. We report our results on FPR-$95$ with further results on AUROC and calibration in Supplementary Section~\ref{sup:results}. FPR-${95}$ is the FPR when the TPR is $95$\%. The metric can be interpreted as the probability that a negative sample will be classified as positive when $95\%$ of samples are correctly classified as positive. A lower score is better.

\textbf{Baselines.} We compare PreLoad against baseline methods in two settings: trained from scratch and fine-tuned. In the former, we compare against a DNN trained on in-domain data, referred to as Standard, OOD training baselines including a ``None'' class method~\citep{kristiadi2022being} referred to as NC, Outlier Exposure~\citep{hendrycks2018deep} referred to as OE and a generative modeling baseline \citep{mukhoti2023deep} referred to as DDU. All the methods are developed starting from identical neural network architectures and we select optimal hyper-parameters for PreLoad based on maximizing in-distribution validation accuracy. Standard, OE, NC, and PreLoad are all trained for 100 epochs from scratch. DDU trains a Gaussian Mixture Model over the Standard method for OOD detection. Finetuned (FT) baselines include OE (OE-FT), DOE (DOE-FT)  and Energy-FT~\citep{liu2020energy}. PreLoad-FT and the FT baselines are initialized from a Standard model and are fine-tuned over 10 epochs using the respective losses. Training details can be found in Supplementary Section~\ref{sup:training}.

\begin{figure*}
  \centering
  \includegraphics[width=0.495\linewidth]{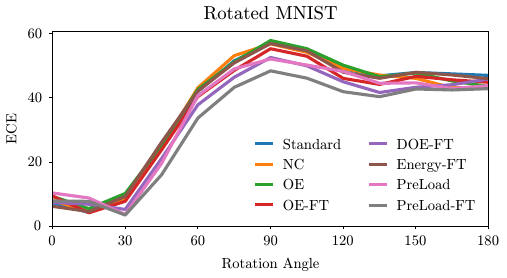}
  \includegraphics[width=0.495\linewidth]{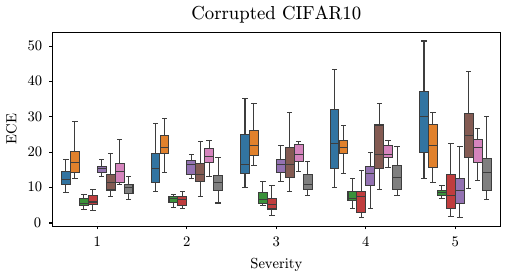}

  \caption{Calibration results, measured on the ECE metric, on Rotated MNIST and CIFAR10-C following ~\cite{ovadia2019can}.
  }
  \label{fig:shift_ece}
\end{figure*}

\subsection{Far-Away Data}

We present results on two types of far-away data: FarAway~\citep{hein2019relu} and FarAway Random Direction (FarAway-RD). A FarAway sample $s$ is defined as $s = t u_x$, where $u_x$ has the shape of a training sample $x$ and contains values sampled from a uniform distribution on the interval [0, 1) and $t$ is some constant. On the other hand, a FarAway-RD sample $s'$ can be defined as $s' = u_x + t v$, where $u_x$ and $t$ are as previously defined and $v$ is a sample from the unit sphere such that $\Vert v \Vert=1$. As the name suggests, Faraway-RD can scale the data in a random direction. For all the experiments we have fixed $t=10,000$.
Note that far-away data defined as such is unbounded, i.e., in $\R^n$, whereas realistic images are in $[0,1]^n$. In the subsequent section we present results on realistic images.


Table~\ref{tab:faraway-fpr95} shows the FPR-$95$ metric when we evaluate on the two types of far-away  data for models trained on each dataset. We observe that both versions of our method, PreLoad, and PreLoad-FT, achieve perfect FPR-$95$ of $0$ on all the datasets on both types of far-away data. The Standard method is the worst with OE and DOE-FT also doing poorly. Energy-FT, which incorporates an energy function into the loss also does not do well in this setting. NC performs better in some scenarios such as on FarAway on MNIST but generally has high variability between different runs. DDU, which trains a Gaussian Mixture Model on top of the neural network embedding, unsurprisingly, achieves perfect results as well.

\subsection{OOD Benchmarks}

Next, we present our results on standard OOD benchmarks which evaluate a more realistic scenario for the  evaluation of an image classifier. Models trained on MNIST and F-MNIST are evaluated on each other and E-MNIST~\citep{cohen2017emnist}, K-MNIST~\citep{clanuwat2018deep} and grey-scale CIFAR (CIFAR-Gr). Models trained on SVHN, CIFAR10 and CIFAR100 are evaluated on each other as well as LSUN classroom (LSUN-CR)~\citep{yu2015lsun} and Fashion MNIST 3D (FMNIST-3D). Additionally, all models are evaluated on uniform noise shaped like the relevant images and smooth noise~\citep{hein2019relu}, obtained by permuting, blurring and contrast-rescaling the original training data. Further information is provided in the Supplementary Section~\ref{sup:data}.

Table~\ref{tab:fpr_avg} presents the FPR-$95$ averaged over all the OOD evaluation sets with error bars. Detailed results are in Section~\ref{sup:results}. We have indicated in bold the best ``Trained from scratch'' and best ``Finetuned'' results in each row. Note that we take into account the error bars when highlighting the best results. 

When training from scratch, PreLoad along with NC performs the best on F-MNIST, SHVN, CIFAR10, and CIFAR100. On MNIST, the NC method and OE perform better. Note that DDU which can prevent arbitrarily high confidence on far-away data is always significantly worse than our method on realistic OOD data. In the FT setting, we observe that Energy-FT and DOE-FT are better than OE-FT, however, our FT method performs the best.

We note that extra class methods, such as NC and ours, use the confidence of the additional class to detect OOD, unlike other methods such as Standard or OE which use $\max P(y|x)$ amongst all the classes. Since the additional class is trained on OOD data, such methods tend to perform better on OOD detection.

\subsection{Dataset Shifts}

Once we have established that our method performs well on far-away and realistic OOD data, we evaluate model calibration under data shift. Calibration is an important measure of uncertainty quantification. We evaluate calibration using the ECE (confidence ECE following \citealt{guo2017calibration}) metric with $15$ bins. We use Rotated MNIST~\citep{ovadia2019can} and Corrupted CIFAR10 (CIFAR10-C)~\citep{hendrycks2018benchmarking} for evaluating on data shift.

We observe in Figure~\ref{fig:shift_ece} that on Rotated MNIST, as we increase the rotation angle, PreLoad-FT performs the best followed by DOE-FT and PreLoad. PreLoad-FT scores the lowest ECE when the angle moves beyond $30$. On CIFAR10-C we observe that as we increase corruption severity, ECE for Standard degrades the most followed by NC and Energy-FT. The OE and OE-FT methods perform the best followed by PreLoad-FT and DOE-FT. We observe that the FT methods do better than the methods trained from scratch.

\cite{kristiadi2022being} suggest that NC, which is an extra class method, may do worse on dataset shift as it uses the confidence of the additional class which is trained on OOD data. Corrupted data may resemble OOD and therefore the calibration would be off. Our method on the other hand demonstrates that carefully designed extra-class methods can be better calibrated under dataset shift.

\section{CONCLUSION}
\label{sec:conclusion}
In this work, we have presented PreLoad, an OOD detection method that provably fixes arbitrarily high confidence in neural networks on far-away data.
PreLoad works by training an extra class which produces larger logits as test samples move farther from the training data. Unlike all other baselines, PreLoad fulfills each of our three desiderata: a) maintain the simplicity of standard discriminative training b) provably fix arbitrarily high confidence on far-away data and c) perform well on realistic OOD samples. Future work could include training PreLoad with perturbed data such as adversarial examples, and adapting it to OOD detection in language models. 

\section*{Acknowledgements}

This research was partly funded by the Waterloo Apple PhD Fellowship, Natural Sciences and Engineering Council of Canada, and the David R. Cheriton Graduate Scholarship. Resources used in this work were provided by Huawei Canada, the Province of Ontario, the Government of Canada through CIFAR, companies sponsoring the Vector Institute \url{https://vectorinstitute.ai/partners/} and the Natural Sciences and Engineering Council of Canada.



\bibliography{aistats2024}

\begin{thebibliography}{39}
\providecommand{\natexlab}[1]{#1}
\providecommand{\url}[1]{\texttt{#1}}
\expandafter\ifx\csname urlstyle\endcsname\relax
  \providecommand{\doi}[1]{doi: #1}\else
  \providecommand{\doi}{doi: \begingroup \urlstyle{rm}\Url}\fi

\bibitem[Arora et~al.(2018)Arora, Basu, Mianjy, and Mukherjee]{arora2018understanding}
Arora, Raman, Basu, Amitabh, Mianjy, Poorya, and Mukherjee, Anirbit.
\newblock Understanding deep {N}eural {N}etworks with rectified linear units.
\newblock In \emph{International Conference on Learning Representations}, 2018.

\bibitem[Blundell et~al.(2015)Blundell, Cornebise, Kavukcuoglu, and Wierstra]{blundell2015weight}
Blundell, Charles, Cornebise, Julien, Kavukcuoglu, Koray, and Wierstra, Daan.
\newblock Weight uncertainty in {N}eural {N}etwork.
\newblock In \emph{International Conference on Machine Learning}. PMLR, 2015.

\bibitem[Clanuwat et~al.(2018)Clanuwat, Bober-Irizar, Kitamoto, Lamb, Yamamoto, and Ha]{clanuwat2018deep}
Clanuwat, Tarin, Bober-Irizar, Mikel, Kitamoto, Asanobu, Lamb, Alex, Yamamoto, Kazuaki, and Ha, David.
\newblock Deep learning for classical japanese literature.
\newblock \emph{arXiv preprint arXiv:1812.01718}, 2018.

\bibitem[Cohen et~al.(2017)Cohen, Afshar, Tapson, and Van~Schaik]{cohen2017emnist}
Cohen, Gregory, Afshar, Saeed, Tapson, Jonathan, and Van~Schaik, Andre.
\newblock Emnist: Extending {MNIST} to handwritten letters.
\newblock In \emph{2017 international joint conference on {N}eural {N}etworks}. IEEE, 2017.

\bibitem[Gal and Ghahramani(2016)]{gal2016dropout}
Gal, Yarin and Ghahramani, Zoubin.
\newblock {D}ropout as a {B}ayesian approximation: Representing model uncertainty in deep learning.
\newblock In \emph{International Conference on Machine Learning}. PMLR, 2016.

\bibitem[Goodfellow et~al.(2014)Goodfellow, Pouget-Abadie, Mirza, Xu, Warde-Farley, Ozair, Courville, and Bengio]{goodfellow2014generative}
Goodfellow, Ian, Pouget-Abadie, Jean, Mirza, Mehdi, Xu, Bing, Warde-Farley, David, Ozair, Sherjil, Courville, Aaron, and Bengio, Yoshua.
\newblock Generative adversarial nets.
\newblock \emph{Advances in Neural Information Processing Systems}, 27, 2014.

\bibitem[Guo et~al.(2017)Guo, Pleiss, Sun, and Weinberger]{guo2017calibration}
Guo, Chuan, Pleiss, Geoff, Sun, Yu, and Weinberger, Kilian~Q.
\newblock On calibration of modern {N}eural {N}etworks.
\newblock In \emph{International Conference on Machine Learning}. PMLR, 2017.

\bibitem[Gupta et~al.(2020)Gupta, Rahimi, Ajanthan, Mensink, Sminchisescu, and Hartley]{gupta2020calibration}
Gupta, Kartik, Rahimi, Amir, Ajanthan, Thalaiyasingam, Mensink, Thomas, Sminchisescu, Cristian, and Hartley, Richard.
\newblock Calibration of {N}eural {N}etworks using {S}plines.
\newblock In \emph{International Conference on Learning Representations}, 2020.

\bibitem[Hein et~al.(2019)Hein, Andriushchenko, and Bitterwolf]{hein2019relu}
Hein, Matthias, Andriushchenko, Maksym, and Bitterwolf, Julian.
\newblock Why {ReLU} networks yield high-confidence predictions far away from the training data and how to mitigate the problem.
\newblock In \emph{Proceedings of the IEEE/CVF Conference on Computer Vision and Pattern Recognition}, 2019.

\bibitem[Hendrycks and Dietterich(2018)]{hendrycks2018benchmarking}
Hendrycks, Dan and Dietterich, Thomas.
\newblock Benchmarking {N}eural {N}etwork robustness to common corruptions and perturbations.
\newblock In \emph{International Conference on Learning Representations}, 2018.

\bibitem[Hendrycks et~al.(2018)Hendrycks, Mazeika, and Dietterich]{hendrycks2018deep}
Hendrycks, Dan, Mazeika, Mantas, and Dietterich, Thomas.
\newblock Deep anomaly detection with {O}utlier {E}xposure.
\newblock In \emph{International Conference on Learning Representations}, 2018.

\bibitem[Kingma and Welling(2013)]{kingma2013auto}
Kingma, Diederik~P and Welling, Max.
\newblock Auto-encoding variational {B}ayes.
\newblock \emph{arXiv preprint arXiv:1312.6114}, 2013.

\bibitem[Kristiadi et~al.(2020)Kristiadi, Hein, and Hennig]{kristiadi2020being}
Kristiadi, Agustinus, Hein, Matthias, and Hennig, Philipp.
\newblock Being {B}ayesian, even just a bit, fixes overconfidence in {ReLU} {N}etworks.
\newblock In \emph{International Conference on Machine Learning}. PMLR, 2020.

\bibitem[Kristiadi et~al.(2021)Kristiadi, Hein, and Hennig]{kristiadi2021infinite}
Kristiadi, Agustinus, Hein, Matthias, and Hennig, Philipp.
\newblock An infinite-feature extension for {B}ayesian {ReLU} nets that fixes their asymptotic overconfidence.
\newblock \emph{Advances in Neural Information Processing Systems}, 34, 2021.

\bibitem[Kristiadi et~al.(2022{\natexlab{a}})Kristiadi, Eschenhagen, and Hennig]{kristiadi2022posterior}
Kristiadi, Agustinus, Eschenhagen, Runa, and Hennig, Philipp.
\newblock Posterior refinement improves sample efficiency in {B}ayesian {N}eural {N}etworks.
\newblock 2022{\natexlab{a}}.

\bibitem[Kristiadi et~al.(2022{\natexlab{b}})Kristiadi, Hein, and Hennig]{kristiadi2022being}
Kristiadi, Agustinus, Hein, Matthias, and Hennig, Philipp.
\newblock Being a bit frequentist improves bayesian {N}eural {N}etworks.
\newblock In \emph{International Conference on Artificial Intelligence and Statistics}. PMLR, 2022{\natexlab{b}}.

\bibitem[Krizhevsky et~al.(2009)]{krizhevsky2009learning}
Krizhevsky, Alex et~al.
\newblock Learning multiple layers of features from tiny images.
\newblock 2009.

\bibitem[Kull et~al.(2019)Kull, Perello~Nieto, K{\"a}ngsepp, Silva~Filho, Song, and Flach]{kull2019beyond}
Kull, Meelis, Perello~Nieto, Miquel, K{\"a}ngsepp, Markus, Silva~Filho, Telmo, Song, Hao, and Flach, Peter.
\newblock Beyond temperature scaling: Obtaining well-calibrated multi-class probabilities with {D}irichlet calibration.
\newblock \emph{Advances in neural information processing systems}, 32, 2019.

\bibitem[Kumar et~al.(2018)Kumar, Sarawagi, and Jain]{kumar2018trainable}
Kumar, Aviral, Sarawagi, Sunita, and Jain, Ujjwal.
\newblock Trainable calibration measures for {N}eural {N}etworks from kernel mean embeddings.
\newblock In \emph{International Conference on Machine Learning}. PMLR, 2018.

\bibitem[Lakshminarayanan et~al.(2017)Lakshminarayanan, Pritzel, and Blundell]{lakshminarayanan2017simple}
Lakshminarayanan, Balaji, Pritzel, Alexander, and Blundell, Charles.
\newblock Simple and scalable predictive uncertainty estimation using {D}eep {E}nsembles.
\newblock \emph{Advances in neural information processing systems}, 30, 2017.

\bibitem[LeCun et~al.(1998)LeCun, Bottou, Bengio, and Haffner]{lecun1998gradient}
LeCun, Yann, Bottou, L{\'e}on, Bengio, Yoshua, and Haffner, Patrick.
\newblock Gradient-based learning applied to document recognition.
\newblock \emph{Proceedings of the IEEE}, 86\penalty0 (11), 1998.

\bibitem[Lee et~al.(2018)Lee, Lee, Lee, and Shin]{lee2018simple}
Lee, Kimin, Lee, Kibok, Lee, Honglak, and Shin, Jinwoo.
\newblock A simple unified framework for detecting out-of-distribution samples and {A}dversarial {A}ttacks.
\newblock \emph{Advances in Neural Information Processing Systems}, 31, 2018.

\bibitem[Lin et~al.(2017)Lin, Goyal, Girshick, He, and Doll{\'a}r]{lin2017focal}
Lin, Tsung-Yi, Goyal, Priya, Girshick, Ross, He, Kaiming, and Doll{\'a}r, Piotr.
\newblock {F}ocal loss for dense object detection.
\newblock In \emph{Proceedings of the IEEE international conference on computer vision}, 2017.

\bibitem[Liu et~al.(2020)Liu, Wang, Owens, and Li]{liu2020energy}
Liu, Weitang, Wang, Xiaoyun, Owens, John, and Li, Yixuan.
\newblock Energy-based out-of-distribution detection.
\newblock \emph{Advances in Neural Information Processing Systems}, 33, 2020.

\bibitem[Louizos and Welling()]{louizos2017multiplicative}
Louizos, Christos and Welling, Max.
\newblock Multiplicative normalizing flows for variational {B}ayesian {N}eural {N}etworks.
\newblock In \emph{International Conference on Machine Learning}.

\bibitem[Meinke and Hein(2020)]{meinke2020towards}
Meinke, Alexander and Hein, Matthias.
\newblock Towards {N}eural {N}etworks that provably know when they don't know.
\newblock In \emph{International Conference on Learning Representations}, 2020.

\bibitem[Minderer et~al.(2021)Minderer, Djolonga, Romijnders, Hubis, Zhai, Houlsby, Tran, and Lucic]{NEURIPS2021_8420d359}
Minderer, Matthias, Djolonga, Josip, Romijnders, Rob, Hubis, Frances, Zhai, Xiaohua, Houlsby, Neil, Tran, Dustin, and Lucic, Mario.
\newblock Revisiting the calibration of modern {N}eural {N}etworks.
\newblock In Ranzato, M., Beygelzimer, A., Dauphin, Y., Liang, P.S., and Vaughan, J.~Wortman, editors, \emph{Advances in Neural Information Processing Systems}, volume~34. Curran Associates, Inc., 2021.

\bibitem[Mukhoti et~al.(2023)Mukhoti, Kirsch, van Amersfoort, Torr, and Gal]{mukhoti2023deep}
Mukhoti, Jishnu, Kirsch, Andreas, van Amersfoort, Joost, Torr, Philip~HS, and Gal, Yarin.
\newblock Deep deterministic uncertainty: A new simple baseline.
\newblock In \emph{Proceedings of the IEEE/CVF Conference on Computer Vision and Pattern Recognition}, 2023.

\bibitem[M{\"u}ller et~al.(2019)M{\"u}ller, Kornblith, and Hinton]{muller2019does}
M{\"u}ller, Rafael, Kornblith, Simon, and Hinton, Geoffrey~E.
\newblock When does {L}abel {S}moothing help?
\newblock \emph{Advances in neural information processing systems}, 32, 2019.

\bibitem[Netzer et~al.(2011)Netzer, Wang, Coates, Bissacco, Wu, and Ng]{netzer2011reading}
Netzer, Yuval, Wang, Tao, Coates, Adam, Bissacco, Alessandro, Wu, Bo, and Ng, Andrew~Y.
\newblock Reading digits in natural images with unsupervised feature learning.
\newblock 2011.

\bibitem[Nguyen et~al.(2015)Nguyen, Yosinski, and Clune]{nguyen2015deep}
Nguyen, Anh, Yosinski, Jason, and Clune, Jeff.
\newblock Deep {N}eural {N}etworks are easily fooled: High confidence predictions for unrecognizable images.
\newblock In \emph{Proceedings of the IEEE conference on computer vision and pattern recognition}, 2015.

\bibitem[Ovadia et~al.(2019)Ovadia, Fertig, Ren, Nado, Sculley, Nowozin, Dillon, Lakshminarayanan, and Snoek]{ovadia2019can}
Ovadia, Yaniv, Fertig, Emily, Ren, Jie, Nado, Zachary, Sculley, David, Nowozin, Sebastian, Dillon, Joshua, Lakshminarayanan, Balaji, and Snoek, Jasper.
\newblock Can you trust your model's uncertainty? evaluating predictive uncertainty under dataset shift.
\newblock \emph{Advances in Neural Information Processing Systems}, 32, 2019.

\bibitem[Thulasidasan et~al.(2019)Thulasidasan, Chennupati, Bilmes, Bhattacharya, and Michalak]{thulasidasan2019mixup}
Thulasidasan, Sunil, Chennupati, Gopinath, Bilmes, Jeff~A, Bhattacharya, Tanmoy, and Michalak, Sarah.
\newblock On {M}ixup training: Improved calibration and predictive uncertainty for deep {N}eural {N}etworks.
\newblock \emph{Advances in Neural Information Processing Systems}, 32, 2019.

\bibitem[Torralba et~al.(2008)Torralba, Fergus, and Freeman]{torralba200880}
Torralba, Antonio, Fergus, Rob, and Freeman, William~T.
\newblock 80 million tiny images: A large data set for nonparametric object and scene recognition.
\newblock \emph{IEEE transactions on Pattern Analysis and Machine Intelligence}, 30\penalty0 (11), 2008.

\bibitem[Wang et~al.(2023)Wang, Ye, Liu, Dai, Kalander, Liu, Hao, and Han]{wang2023outofdistribution}
Wang, Qizhou, Ye, Junjie, Liu, Feng, Dai, Quanyu, Kalander, Marcus, Liu, Tongliang, Hao, Jianye, and Han, Bo.
\newblock Out-of-distribution detection with implicit outlier transformation.
\newblock In \emph{International Conference on Learning Representations}, 2023.
\newblock URL \url{https://openreview.net/forum?id=hdghx6wbGuD}.

\bibitem[Xiao et~al.(2017)Xiao, Rasul, and Vollgraf]{xiao2017fashion}
Xiao, Han, Rasul, Kashif, and Vollgraf, Roland.
\newblock {Fashion-MNIST}: a novel image dataset for benchmarking machine learning algorithms.
\newblock \emph{arXiv preprint arXiv:1708.07747}, 2017.

\bibitem[Yu et~al.(2015)Yu, Seff, Zhang, Song, Funkhouser, and Xiao]{yu2015lsun}
Yu, Fisher, Seff, Ari, Zhang, Yinda, Song, Shuran, Funkhouser, Thomas, and Xiao, Jianxiong.
\newblock {LSUN}: Construction of a large-scale image dataset using deep learning with humans in the loop.
\newblock \emph{arXiv preprint arXiv:1506.03365}, 2015.

\bibitem[Zagoruyko and Komodakis(2016)]{zagoruyko2016wide}
Zagoruyko, Sergey and Komodakis, Nikos.
\newblock Wide {R}esidual {N}etworks.
\newblock \emph{arXiv preprint arXiv:1605.07146}, 2016.

\bibitem[Zhang and LeCun(2017)]{zhang2017universum}
Zhang, Xiang and LeCun, Yann.
\newblock Universum prescription: Regularization using unlabeled data.
\newblock In \emph{Proceedings of the AAAI Conference on Artificial Intelligence}, volume~31, 2017.

\end{thebibliography}
\bibliographystyle{revplainnat}

\section*{Checklist}



 \begin{enumerate}

 \item For all models and algorithms presented, check if you include:
 \begin{enumerate}
   \item A clear description of the mathematical setting, assumptions, algorithm, and/or model. [Yes]
   \item An analysis of the properties and complexity (time, space, sample size) of any algorithm. [Not Applicable]
   \item (Optional) Anonymized source code, with specification of all dependencies, including external libraries. [Yes]
 \end{enumerate}

 \item For any theoretical claim, check if you include:
 \begin{enumerate}
   \item Statements of the full set of assumptions of all theoretical results. [Yes]
   \item Complete proofs of all theoretical results. [Yes]
   \item Clear explanations of any assumptions. [Yes]
 \end{enumerate}

 \item For all figures and tables that present empirical results, check if you include:
 \begin{enumerate}
   \item The code, data, and instructions needed to reproduce the main experimental results (either in the supplemental material or as a URL). [Yes]
   \item All the training details (e.g., data splits, hyperparameters, how they were chosen). [Yes]
         \item A clear definition of the specific measure or statistics and error bars (e.g., with respect to the random seed after running experiments multiple times). [Yes]
         \item A description of the computing infrastructure used. (e.g., type of GPUs, internal cluster, or cloud provider). [Yes]
 \end{enumerate}

 \item If you are using existing assets (e.g., code, data, models) or curating/releasing new assets, check if you include:
 \begin{enumerate}
   \item Citations of the creator If your work uses existing assets. [Yes]
   \item The license information of the assets, if applicable. [Yes]
   \item New assets either in the supplemental material or as a URL, if applicable. [Not Applicable]
   \item Information about consent from data providers/curators. [Not Applicable]
   \item Discussion of sensible content if applicable, e.g., personally identifiable information or offensive content. [Not Applicable]
 \end{enumerate}

 \item If you used crowdsourcing or conducted research with human subjects, check if you include:
 \begin{enumerate}
   \item The full text of instructions given to participants and screenshots. [Not Applicable]
   \item Descriptions of potential participant risks, with links to Institutional Review Board (IRB) approvals if applicable. [Not Applicable]
   \item The estimated hourly wage paid to participants and the total amount spent on participant compensation. [Not Applicable]
 \end{enumerate}

 \end{enumerate}

 \newpage

\begin{appendices}
  \onecolumn
    \aistatstitle{Preventing Arbitrarily High Confidence on Far-Away Data in Point-Estimated Discriminative Neural Networks \\ Supplementary Materials}
    \thispagestyle{empty}

  \section{Experimental Details}
  \label{sec:sup}
  \subsection{Training Details}
\label{sup:training}

Our code is publicly available at: \url{https://github.com/serenahacker/PreLoad}.

We trained all models on a single 12GB-Tesla P-100 GPU. All results are averaged over 5 random seeds.
Models were either trained from scratch or fine-tuned.
When trained from scratch, all models were trained for 100 epochs and when fine-tuned, the Standard model was trained for 100 epochs with a further 10 epochs of fine-tuning. Note that OE-FT and Energy-FT do not introduce any new parameters to the Standard network, so all the parameters are sufficiently initialized during the 100 epochs of pre-training. On the contrary, PreLoad-FT introduces additional parameters, \(w_{k+1}\) and \(b_{k+1}\). In order to initialize them properly and maintain a fair comparison between the algorithms, we pre-train the new weights using the objective $\min_{w_{k+1},b_{k+1}}\mathcal{R}(f_{\psi,W}(x))$ for 10 epochs. All other parameters are frozen. After that, we fine-tune all the parameters $\theta$ using the objective  $\min_\theta \mathcal{R}(f_{\psi,W}(x))$ where $\theta = \{\psi, W \}$ as in \cref{alg:method}---note that \(W\) here also includes \(w_{k+1}\) and \(b_{k+1}\).

Note that in our implementation, we restrict the weights $w_{k+1}$ to the positive orthant by setting $w_{k+1} = \exp(w'_{k+1})$ for some weights $w'_{k+1}$.



In all experiments, we used a batch size of 128 and a Cosine Annealing Scheduler for the learning rate. We tuned some of the hyper-parameters for PreLoad, PreLoad-FT and Energy-FT using WandB\footnote{\url{https://github.com/wandb/wandb}} sweeps. Each sweep consisted of 50 runs. Optimal hyper-parameter values were selected based on the highest evaluation accuracy. Table~\ref{tab:param-tuning} lists our tuning strategy for each algorithm. Note that wd stands for weight decay, lr is learning rate, $m_\text{in}$ and $m_\text{ood}$ are the in-domian and OOD margin parameters for Energy-FT, and $\lambda$ is a constant that scales the OOD loss. Note the for Energy-FT the OOD loss has both an in-domain and an out-of-domain component. Also note that for DOE-FT, $\lambda$ controls the relative weight of the cross-entropy loss and the OE loss~\citep{wang2023outofdistribution}.

Tables~\ref{tab:param_mnist} to \ref{tab:param_cifar100} list the important hyper-parameters for all the methods for all the data sets. The values in bold were obtained after hyper-parameter tuning. For DOE-FT, we mostly used the same hyper-parameters as specified in ~\cite{wang2023outofdistribution} (i.e. $\beta = 0.6$, $\lambda = 1$, $\mathrm{warmup\ epochs} = 5$, $\mathrm{perturbation\ steps} = 1$). We changed $\alpha$ to be uniformly sampled from $\{1e^{-2}, 1e^{-3}, 1e^{-4}\}$ (as opposed to $\{1e^{-1}, 1e^{-2}, 1e^{-3}, 1e^{-4}\}$) as we found that this resulted in significantly higher validation accuracy for CIFAR-10 and CIFAR-100. Note that for the CIFAR datasets, we use WideResNet-16-4, while the authors of DOE, use WideResNet-40-2.
\vfill

\begin{table}[h]
\centering  

\caption{Hyper-parameter Tuning Strategies}
\label{tab:param-tuning}
\begin{tabular}{l c c }

\midrule

\textbf{Method} & \textbf{Tuning Method} & \textbf{Parameter Ranges} \\

\midrule
\textbf{Energy-FT} &  Random Search & $m_\text{in}$: -30 to 0 \\
& & $m_\text{ood}$: -30 to 0\\
\midrule
\textbf{PreLoad} & Bayesian Optimization & lr: $1.0\times 10^{-3}$ to $5\times 10^{-1}$ \\
& & wd: $1.0 \times 10^{-5}$ to $1.0 \times 10^{-3}$\\
\midrule
\textbf{PreLoad-FT-Init} & Random Search &  $\lambda$: $1.0 \times 10^{-2}$ to 1.0  \\
& & lr: $1.0 \times 10^{-4}$ to $1.0 \times 10^{-1}$  \\
& & wd: $1.0 \times 10^{-6}$ to $1.0$\\
\midrule
\textbf{PreLoad-FT} & Random Search &  $\lambda$: $1.0 \times 10^{-2}$ to 1.0  \\
& & lr: $1.0 \times 10^{-4}$ to $1.0 \times 10^{-1}$\\
& & wd: $1.0 \times 10^{-6}$ to $1.0 \times 10^{-3}$\\

\midrule
\end{tabular}
\end{table}

\begin{table}[h]
  \caption{MNIST Hyper-Parameters}
  \label{tab:param_mnist}
    \resizebox{\linewidth}{!}{%
    \begin{tabular}{l c c c c c c}

\midrule

\textbf{Methods} & \textbf{Optimizer} & \textbf{Learning Rate} & \textbf{Weight Decay}  & \textbf{$\lambda$} & \textbf{In Margin} & \textbf{Out Margin} \\

\midrule

Standard & Adam & $1.0 \times 10^{-3}$ & $5.0 \times 10^{-4}$  & - & - & - \\
NC & Adam & $1.0 \times 10^{-3}$ & $5.0 \times 10^{-4}$ &  - & - & -  \\
OE & Adam & $1.0 \times 10^{-3}$ & $5.0 \times 10^{-4}$ & $5.0 \times 10^{-1}$ & - & -  \\
Pre-Load & Adam & $1.0 \times 10^{-3}$ & $5.0 \times 10^{-4}$  & $1.0 \times 10^0$ & - & -  \\
OE-FT  & Adam & $1.0 \times 10^{-3}$ & $5.0 \times 10^{-4}$ &  $5.0 \times 10^{-1}$ & - & -  \\
DOE-FT  & Adam & $1.0 \times 10^{-3}$ & $5.0 \times 10^{-4}$ &  $1.0 \times 10^0$  & - & -  \\
Energy-FT & Adam & $1.0 \times 10^{-3}$ & $5.0 \times 10^{-4}$  & $1.0 \times 10^{-1}$ & \textbf{-3.6} & \textbf{-25.0}  \\
PreLoad-FT-Init & Adam & $\mathbf{4.1 \times 10^{-3}}$ & $\mathbf{3.1 \times 10^{-4}}$  & $\mathbf{8.0 \times 10^{-1}}$ & - & -  \\
PreLoad-FT & Adam & $\mathbf{4.1 \times 10^{-3}}$ & $\mathbf{3.1 \times 10^{-4}}$  & $\mathbf{8.0 \times 10^{-1}}$ & - & -  \\
\midrule

\end{tabular}
}
\end{table}

\begin{table}[h]
  \caption{F-MNIST Hyper-Parameters}
  \label{tab:param_fmnist}
    \resizebox{\linewidth}{!}{%
    \begin{tabular}{l c c  c c c c}

\midrule

\textbf{Methods} & \textbf{Optimizer} & \textbf{Learning Rate} & \textbf{Weight Decay}  & \textbf{$\lambda$} & \textbf{In Margin} & \textbf{Out Margin} \\

\midrule

Standard & Adam & $1.0 \times 10^{-3}$ & $5.0 \times 10^{-4}$  & - & - & - \\
NC & Adam & $1.0 \times 10^{-3}$ & $5.0 \times 10^{-4}$  & - & - & -  \\
OE & Adam & $1.0 \times 10^{-3}$ & $5.0 \times 10^{-4}$  & $5.0 \times 10^{-1}$ & - & -  \\
Pre-Load & Adam & $1.0 \times 10^{-3}$ & $5.0 \times 10^{-4}$  & $1.0 \times 10^0$ & - & -  \\
OE-FT  & Adam & $1.0 \times 10^{-3}$ & $5.0 \times 10^{-4}$ &  $5.0 \times 10^{-1}$ & - & -  \\
DOE-FT  & Adam & $1.0 \times 10^{-3}$ & $5.0 \times 10^{-4}$ &  $1.0 \times 10^0$  & - & -  \\
Energy-FT & Adam & $1.0 \times 10^{-3}$ & $5.0 \times 10^{-4}$  & $1.0 \times 10^{-1}$ & $\mathbf{6.4 \times 10^{-2}}$ & \textbf{-4.0}  \\
PreLoad-FT-Init & Adam & $\mathbf{6.0 \times 10^{-3}}$ & $\mathbf{2.5 \times 10^{-3}}$  & $\mathbf{2.6 \times 10^{-2}}$ & - & -  \\
PreLoad-FT & Adam & $\mathbf{6.0 \times 10^{-3}}$ & $\mathbf{2.5 \times 10^{-3}}$  & $\mathbf{2.6 \times 10^{-2}}$ & - & -  \\
\midrule

\end{tabular}
}
\end{table}

\begin{table}
  \caption{SVHN Hyper-Parameters}
  \label{tab:param_svhn}
    \resizebox{\linewidth}{!}{%
    \begin{tabular}{l c c c c c c c}

\midrule

\textbf{Methods} & \textbf{Optimizer} & \textbf{Learning Rate} & \textbf{Weight Decay} & \textbf{Momentum} & \textbf{$\lambda$} & \textbf{In Margin} & \textbf{Out Margin} \\

\midrule

Standard & SGD & $1.0 \times 10^{-1}$ & $5.0 \times 10^{-4}$ & $9.0 \times 10^{-1}$ & - & - & - \\
NC & SGD & $1.0 \times 10^{-1}$ & $5.0 \times 10^{-4}$ & $9.0 \times 10^{-1}$ & - & - & -  \\
OE & SGD & $1.0 \times 10^{-1}$ & $5.0 \times 10^{-4}$ & $9.0 \times 10^{-1}$ & $5.0 \times 10^{-1}$ & - & -  \\
Pre-Load & SGD & $1.0 \times 10^{-1}$ & $5.0 \times 10^{-4}$ & $9.0 \times 10^{-1}$ & $1.0 \times 10^0$ & - & -  \\
OE-FT  & SGD & $1.0 \times 10^{-3}$ & $5.0 \times 10^{-4}$ & $9.0 \times 10^{-1}$ & $5.0 \times 10^-1$ & - & -  \\
DOE-FT  & SGD & $1.0 \times 10^{-3}$ & $5.0 \times 10^{-4}$ & $9.0 \times 10^{-1}$ & $1.0 \times 10^0$  & - & -  \\
Energy-FT & SGD & $1.0 \times 10^{-3}$ & $5.0 \times 10^{-4}$ & $9.0 \times 10^{-1}$ & $1.0 \times 10^{-1}$ & \textbf{-5.7} & \textbf{-12.3}  \\
PreLoad-FT-Init & SGD & $\mathbf{2.9 \times 10^{-2}}$ & $\mathbf{2.8 \times 10^{-6}}$ & $9.0 \times 10^{-1}$ & $\mathbf{2.2 \times 10^{-1}}$ & - & -  \\
PreLoad-FT & SGD & $\mathbf{2.9 \times 10^{-2}}$ & $\mathbf{2.8 \times 10^{-6}}$ & $9.0 \times 10^{-1}$ & $\mathbf{2.2 \times 10^{-1}}$ & - & -  \\
\midrule

\end{tabular}
}
\end{table}

\begin{table}
  \caption{CIFAR-10 Hyper-Parameters}
  \label{tab:param_cifar10}
    \resizebox{\linewidth}{!}{%
    \begin{tabular}{l c c c c c c c}

\midrule

\textbf{Methods} & \textbf{Optimizer} & \textbf{Learning Rate} & \textbf{Weight Decay} & \textbf{Momentum} & \textbf{$\lambda$} & \textbf{In Margin} & \textbf{Out Margin} \\

\midrule

Standard & SGD & $1.0 \times 10^{-1}$ & $5.0 \times 10^{-4}$ & $9.0 \times 10^{-1}$ & - & - & - \\
NC & SGD & $1.0 \times 10^{-1}$ & $5.0 \times 10^{-4}$ & $9.0 \times 10^{-1}$ & - & - & -  \\
OE & SGD & $1.0 \times 10^{-1}$ & $5.0 \times 10^{-4}$ & $9.0 \times 10^{-1}$ & $5.0 \times 10^{-1}$ & - & -  \\
Pre-Load & SGD & $\mathbf{7.3 \times 10^{-2}}$ & $\mathbf{7.6 \times 10^{-4}}$ & $9.0 \times 10^{-1}$ & $1.0 \times 10^0$ & - & -  \\
OE-FT  & SGD & $1.0 \times 10^{-3}$ & $5.0 \times 10^{-4}$ & $9.0 \times 10^{-1}$ & $5.0 \times 10^{-1}$ & - & -  \\
DOE-FT  & SGD & $1.0 \times 10^{-3}$ & $5.0 \times 10^{-4}$ & $9.0 \times 10^{-1}$ & $1.0 \times 10^0$  & - & -  \\
Energy-FT & SGD & $1.0 \times 10^{-3}$ & $5.0 \times 10^{-4}$ & $9.0 \times 10^{-1}$ & $1.0 \times 10^{-1}$ & \textbf{-9.9} & \textbf{-5.7}  \\
PreLoad-FT-Init & SGD & $\mathbf{6.1 \times 10^{-2}}$ & $\mathbf{1.8 \times 10^{-6}}$ & $9.0 \times 10^{-1}$ & $\mathbf{2.0 \times 10^{-2}}$ & - & -  \\
PreLoad-FT & SGD & $\mathbf{6.1 \times 10^{-2}}$ & $\mathbf{1.8 \times 10^{-6}}$ & $9.0 \times 10^{-1}$ & $\mathbf{2.0 \times 10^{-2}}$ & - & -  \\
\midrule

\end{tabular}
}
\end{table}

\begin{table}
  \caption{CIFAR-100 Hyper-Parameters}
  \label{tab:param_cifar100}
    \resizebox{\linewidth}{!}{%
    \begin{tabular}{l c c c c c c c}

\midrule

\textbf{Methods} & \textbf{Optimizer} & \textbf{Learning Rate} & \textbf{Weight Decay} & \textbf{Momentum} & \textbf{$\lambda$} & \textbf{In Margin} & \textbf{Out Margin} \\

\midrule

Standard & SGD & $1.0 \times 10^{-1}$ & $5.0 \times 10^{-4}$ & $9.0 \times 10^{-1}$ & - & - & - \\
NC & SGD & $1.0 \times 10^{-1}$ & $5.0 \times 10^{-4}$ & $9.0 \times 10^{-1}$ & - & - & -  \\
OE & SGD & $1.0 \times 10^{-1}$ & $5.0 \times 10^{-4}$ & $9.0 \times 10^{-1}$ & $5.0 \times 10^{-1}$ & - & -  \\
Pre-Load & SGD & $\mathbf{4.5 \times 10^{-1}}$ & $\mathbf{1.2 \times 10^{-4}}$ & $9.0 \times 10^{-1}$ & $1.0 \times 10^0$ & - & -  \\
OE-FT  & SGD & $1.0 \times 10^{-3}$ & $5.0 \times 10^{-4}$ & $9.0 \times 10^{-1}$ & $5.0 \times 10^{-1}$ & - & -  \\
DOE-FT  & SGD & $1.0 \times 10^{-3}$ & $5.0 \times 10^{-4}$ & $9.0 \times 10^{-1}$ & $1.0 \times 10^0$  & - & -  \\
Energy-FT & SGD & $1.0 \times 10^{-3}$ & $5.0 \times 10^{-4}$ & $9.0 \times 10^{-1}$ & $1.0 \times 10^{-1}$ & \textbf{-14.5} & \textbf{-10.3}  \\
PreLoad-FT-Init & SGD & $\mathbf{6.3 \times 10^{-3}}$ & $\mathbf{1.1 \times 10^{-4}}$ & $9.0 \times 10^{-1}$ & $\mathbf{2.3 \times 10^{-2}}$ & - & -  \\
PreLoad-FT & SGD & $\mathbf{6.3 \times 10^{-3}}$ & $\mathbf{1.1 \times 10^{-4}}$ & $9.0 \times 10^{-1}$ & $\mathbf{2.3 \times 10^{-2}}$ & - & -  \\
\midrule

\end{tabular}
}
\end{table}

\subsection{OOD Test Sets}
\label{sup:data}

In addition to MNIST, F-MNIST, SVHN, CIFAR-10 and CIFAR-100, we evaluate on the following OOD test sets.

\begin{itemize}
    \item E-MNIST consists of handwritten letters and is in the same format as MNIST~\citep{cohen2017emnist}.
    \item K-MNIST consists of handwritten Japanese (Hiragana script) and is in the same format as MNIST~\citep{clanuwat2018deep}.
    \item CIFAR-Gr consists of CIFAR-10 images converted to greyscale.
    \item LSUN-CR consists of real images of classrooms~\citep{yu2015lsun}.
    \item FMNIST-$3$D consists of F-MNIST images converted from single channel to three channels with identical values.

\end{itemize}

  \section{Additional Results}
  \label{sup:results}
  In Table~\ref{tab:complete-fpr95}, we present the complete FPR-$95$ results for all the methods and datasets that were used to compute the average results reported in Table~\ref{tab:fpr_avg}. Note that results for Far-Away and Far-Away-RD are not included in the averages. Additionally, Table~\ref{tab:complete-auroc} presents results with the AUROC metric. AUROC is the area under the receiver operator curve (ROC). The ROC plots the TPR against FPR. AUROC can be interpreted as the probability that a model under test ranks a random positive sample higher than a random negative sample. We report AUROC as a percentage between 0 and 100 where higher the better. Tables~\ref{tab:accuracy} and \ref{tab:calibration} present the accuracy and calibration scores respectively.

\begin{table}
  \caption{FPR-95, Complete Results }
  \label{tab:complete-fpr95}
   \resizebox{\linewidth}{!}{%
        \begin{tabular}{l c c c c c c c c c }

\midrule

\textbf{Datasets} & \textbf{Standard} & \textbf{DDU} & \textbf{NC} & \textbf{OE} & \textbf{PreLoad} & \textbf{OE-FT} & \textbf{DOE-FT} & \textbf{Energy-FT} & \textbf{PreLoad-FT} \\

\midrule

\textbf{MNIST} & & & & & & & & \\
F-MNIST & 8.1$\pm$0.9 & 21.8$\pm$1.1 & 0.0$\pm$0.0 & 0.2$\pm$0.0 & 0.0$\pm$0.0 & 0.3$\pm$0.1 & 0.0$\pm$0.0  & 3.7$\pm$0.7 & 0.1$\pm$0.0 \\
E-MNIST & 32.1$\pm$0.4 & 18.5$\pm$0.3 & 17.0$\pm$0.5 & 27.4$\pm$0.2 & 29.7$\pm$0.4 & 24.8$\pm$0.3 & 22.9$\pm$0.4  & 36.5$\pm$0.3 & 33.1$\pm$2.8 \\
K-MNIST & 10.9$\pm$0.3 & 4.6$\pm$0.1 & 3.1$\pm$0.4 & 5.4$\pm$0.2 & 9.8$\pm$0.6 & 3.4$\pm$0.2 & 2.9$\pm$0.2 & 10.1$\pm$0.6 & 17.2$\pm$3.5 \\
CIFAR-Gr & 0.0$\pm$0.0 & 62.6$\pm$8.2 & 0.0$\pm$0.0 & 0.0$\pm$0.0 & 0.0$\pm$0.0 & 0.0$\pm$0.0 & 0.0$\pm$0.0 & 0.0$\pm$0.0 & 0.0$\pm$0.0 \\
Uniform & 14.1$\pm$7.8 & 81.3$\pm$13.1 & 0.0$\pm$0.0 & 0.0$\pm$0.0 & 0.0$\pm$0.0 & 0.0$\pm$0.0 & 0.0$\pm$0.0 & 0.0$\pm$0.0 & 0.0$\pm$0.0 \\
Smooth & 0.0$\pm$0.0 & 97.5$\pm$2.4 & 0.0$\pm$0.0 & 0.0$\pm$0.0 & 0.0$\pm$0.0 & 0.0$\pm$0.0 & 0.0$\pm$0.0 & 0.0$\pm$0.0 & 0.0$\pm$0.0 \\
FarAway & 100.0$\pm$0.0 & 0.0$\pm$0.0 & 0.0$\pm$0.0 & 56.6$\pm$19.6 & 0.0$\pm$0.0 & 99.0$\pm$0.4 & 56.8$\pm$18.1 & 100.0$\pm$0.0 & 0.0$\pm$0.0 \\
FarAway-RD & 99.9$\pm$0.0 & 0.0$\pm$0.0 & 99.9$\pm$0.1 & 99.8$\pm$0.0 & 0.0$\pm$0.0 & 99.5$\pm$0.1 & 99.8$\pm$0.1 & 100.0$\pm$0.0 & 0.0$\pm$0.0 \\

\midrule

\textbf{F-MNIST} & & & & & & & & \\
MNIST & 74.8$\pm$1.4 & 0.6$\pm$0.2 & 6.8$\pm$0.5 & 65.4$\pm$1.1 & 6.7$\pm$2.1 & 68.4$\pm$1.5 & 51.3$\pm$0.9 & 36.9$\pm$3.5 & 28.8$\pm$3.4 \\
E-MNIST & 72.3$\pm$0.7 & 2.6$\pm$0.6 & 1.7$\pm$0.2 & 55.4$\pm$1.6 & 1.4$\pm$0.3 & 60.4$\pm$0.8 & 32.6$\pm$2.2 & 28.5$\pm$1.7 & 15.2$\pm$3.2 \\
K-MNIST & 73.6$\pm$0.6 & 0.4$\pm$0.1 & 4.7$\pm$0.6 & 58.1$\pm$0.9 & 5.7$\pm$1.2 & 60.6$\pm$1.0 & 38.6$\pm$1.0 & 21.1$\pm$1.3 & 22.6$\pm$2.6 \\
CIFAR-Gr & 85.4$\pm$1.0 & 84.3$\pm$3.8 & 0.0$\pm$0.0 & 0.0$\pm$0.0 & 0.0$\pm$0.0 & 0.1$\pm$0.0 & 0.0$\pm$0.0 & 0.1$\pm$0.0 & 0.0$\pm$0.0 \\
Uniform & 91.9$\pm$2.8 & 33.2$\pm$18.9 & 0.1$\pm$0.1 & 11.3$\pm$9.6 & 0.0$\pm$0.0 & 0.3$\pm$0.2 & 2.0$\pm$0.9 & 0.3$\pm$0.1 & 2.0$\pm$1.3 \\
Smooth & 25.6$\pm$1.4 & 89.6$\pm$0.6 & 0.0$\pm$0.0 & 0.2$\pm$0.0 & 0.0$\pm$0.0 & 0.4$\pm$0.1 & 0.0$\pm$0.0 & 0.2$\pm$0.1 & 5.8$\pm$5.8 \\
FarAway & 100.0$\pm$0.0 & 0.0$\pm$0.0 & 53.5$\pm$22.5 & 100.0$\pm$0.0 & 0.0$\pm$0.0 & 100.0$\pm$0.0 & 99.6$\pm$0.4 & 38.4$\pm$8.9 & 0.0$\pm$0.0 \\
FarAway-RD & 100.0$\pm$0.0 & 0.0$\pm$0.0 & 100.0$\pm$0.0 & 100.0$\pm$0.0 & 0.0$\pm$0.0 & 100.0$\pm$0.0 & 100.0$\pm$0.0 & 81.6$\pm$8.8 & 0.0$\pm$0.0 \\

\midrule

\textbf{SVHN} & & & & & & & & \\
CIFAR-10 & 20.2$\pm$0.6 & 8.3$\pm$0.5 & 0.0$\pm$0.0 & 0.0$\pm$0.0 & 0.0$\pm$0.0 & 0.1$\pm$0.0 & 0.0$\pm$0.0 & 4.7$\pm$0.2 & 0.0$\pm$0.0 \\
LSUN-CR & 24.8$\pm$0.9 & 2.5$\pm$0.5 & 0.0$\pm$0.0 & 0.0$\pm$0.0 & 0.0$\pm$0.0 & 0.0$\pm$0.0 & 0.0$\pm$0.0 & 3.5$\pm$0.3 & 0.0$\pm$0.0 \\
CIFAR-100 & 23.4$\pm$0.6 & 8.9$\pm$0.4 & 0.0$\pm$0.0 & 0.1$\pm$0.0 & 0.0$\pm$0.0 & 0.5$\pm$0.0  & 0.1$\pm$0.0 & 7.5$\pm$0.2 & 0.0$\pm$0.0 \\
FMNIST-3D & 26.5$\pm$0.4 & 25.8$\pm$2.2 & 0.0$\pm$0.0 & 0.0$\pm$0.0 & 0.0$\pm$0.0 & 0.6$\pm$0.3  & 0.0$\pm$0.0 & 14.9$\pm$1.1 & 0.0$\pm$0.0 \\
Uniform & 33.5$\pm$3.9 & 0.0$\pm$0.0 & 0.0$\pm$0.0 & 0.0$\pm$0.0 & 0.0$\pm$0.0 & 0.0$\pm$0.0  & 0.0$\pm$0.0 & 0.9$\pm$0.1 & 0.0$\pm$0.0 \\
Smooth & 14.0$\pm$0.5 & 2.5$\pm$0.4 & 12.6$\pm$1.2 & 10.0$\pm$0.3 & 6.8$\pm$1.0 & 8.6$\pm$0.3  & 7.4$\pm$0.4 & 11.8$\pm$1.3 & 4.6$\pm$0.2 \\
FarAway & 99.4$\pm$0.2 & 0.0$\pm$0.0 & 80.0$\pm$20.0 & 99.4$\pm$0.4 & 0.0$\pm$0.0 & 99.3$\pm$0.3 & 99.9$\pm$0.1 & 100.0$\pm$0.0 & 0.0$\pm$0.0 \\
FarAway-RD & 92.8$\pm$2.1 & 0.0$\pm$0.0 & 80.0$\pm$20.0 & 85.4$\pm$7.6 & 0.0$\pm$0.0 & 93.1$\pm$2.5 & 99.3$\pm$0.6 & 100.0$\pm$0.0 & 0.0$\pm$0.0 \\

\midrule

\textbf{CIFAR-10} & & & & & & & & \\
SVHN & 42.1$\pm$6.5 & 45.1$\pm$2.4 & 0.5$\pm$0.1 & 2.7$\pm$0.6 & 0.4$\pm$0.1 & 8.1$\pm$2.2 & 3.1$\pm$0.5 & 2.9$\pm$0.6 & 4.4$\pm$2.1 \\
LSUN-CR & 50.1$\pm$1.1 & 64.1$\pm$1.7 & 0.6$\pm$0.1 & 6.1$\pm$0.3 & 0.7$\pm$0.2 & 20.7$\pm$1.2 & 11.5$\pm$0.8 & 7.9$\pm$0.3 & 4.3$\pm$0.2 \\
CIFAR-100 & 58.8$\pm$0.6 & 71.0$\pm$0.4 & 26.5$\pm$0.1 & 33.4$\pm$0.3 & 27.3$\pm$0.1 & 44.9$\pm$0.4  & 37.1$\pm$0.4 & 34.0$\pm$0.3 & 35.9$\pm$0.2 \\
FMNIST-3D & 38.9$\pm$1.1 & 35.7$\pm$5.4 & 3.5$\pm$0.4 & 11.5$\pm$0.6 & 3.9$\pm$0.4 & 15.8$\pm$1.0 & 12.3$\pm$0.9 & 8.2$\pm$0.6 & 7.9$\pm$0.8 \\
Uniform & 76.3$\pm$15.5 & 0.0$\pm$0.0 & 0.0$\pm$0.0 & 0.0$\pm$0.0 & 0.0$\pm$0.0 & 0.0$\pm$0.0 & 0.0$\pm$0.0 & 19.8$\pm$19.6 & 0.0$\pm$0.0 \\
Smooth & 40.4$\pm$5.1 & 12.4$\pm$1.8 & 3.4$\pm$0.7 & 16.3$\pm$3.0 & 3.6$\pm$0.6 & 30.2$\pm$1.5 & 26.4$\pm$4.1 & 20.6$\pm$2.6 & 19.4$\pm$6.0 \\
FarAway & 100.0$\pm$0.0 & 0.0$\pm$0.0 & 20.0$\pm$20.0 & 100.0$\pm$0.0 & 0.0$\pm$0.0 & 100.0$\pm$0.0 & 100.0$\pm$0.0 & 100.0$\pm$0.0 & 0.0$\pm$0.0 \\
FarAway-RD & 99.7$\pm$0.2 & 0.0$\pm$0.0 & 40.0$\pm$24.5 & 100.0$\pm$0.0 & 0.0$\pm$0.0 & 99.5$\pm$0.3  & 99.6$\pm$0.4 & 100.0$\pm$0.0 & 0.0$\pm$0.0 \\

\midrule

\textbf{CIFAR-100} & & & & & & & & \\
SVHN & 71.0$\pm$1.1 & 82.4$\pm$4.1 & 21.9$\pm$5.3 & 60.2$\pm$3.0 & 19.0$\pm$2.6 & 68.2$\pm$1.8 & 58.6$\pm$6.4 & 50.0$\pm$6.1 & 46.1$\pm$6.5 \\
LSUN-CR & 78.1$\pm$0.7 & 88.1$\pm$0.6 & 12.1$\pm$0.8 & 62.8$\pm$2.1 & 17.7$\pm$0.6 & 71.7$\pm$0.8 & 39.6$\pm$1.1 & 48.9$\pm$1.1 & 36.6$\pm$0.7 \\
CIFAR-10 & 79.2$\pm$0.4 & 92.8$\pm$0.1 & 79.4$\pm$0.3 & 80.7$\pm$0.4 & 80.5$\pm$0.6 & 80.1$\pm$0.6 & 85.8$\pm$0.5 & 79.7$\pm$0.3 & 88.9$\pm$0.1 \\
FMNIST-3D & 65.6$\pm$2.2 & 90.2$\pm$0.9 & 7.2$\pm$0.5 & 56.0$\pm$3.0 & 13.9$\pm$1.5 & 61.3$\pm$1.6 & 32.0$\pm$1.6 & 34.4$\pm$3.6 & 46.4$\pm$2.8 \\
Uniform & 96.6$\pm$3.3 & 0.0$\pm$0.0 & 0.0$\pm$0.0 & 41.0$\pm$22.1 & 0.0$\pm$0.0 & 75.8$\pm$14.6 & 42.2$\pm$16.3 & 22.7$\pm$19.1 & 0.0$\pm$0.0 \\
Smooth & 72.7$\pm$1.3 & 42.5$\pm$1.2 & 44.1$\pm$4.9 & 60.4$\pm$3.4 & 24.2$\pm$4.1 & 66.6$\pm$1.2 & 67.7$\pm$5.3 & 61.0$\pm$2.7 & 19.0$\pm$3.7 \\
FarAway & 100.0$\pm$0.0 & 0.0$\pm$0.0 & 20.0$\pm$20.0 & 100.0$\pm$0.0 & 0.0$\pm$0.0 & 100.0$\pm$0.0  & 100.0$\pm$0.0 & 100.0$\pm$0.0 & 0.0$\pm$0.0 \\
FarAway-RD & 100.0$\pm$0.0 & 0.0$\pm$0.0 & 20.0$\pm$20.0 & 100.0$\pm$0.0 & 0.0$\pm$0.0 & 100.0$\pm$0.0  & 100.0$\pm$0.0 & 100.0$\pm$0.0 & 0.0$\pm$0.0 \\

\midrule

\end{tabular}

    }
\end{table}

\begin{table}
  \caption{AUROC, Complete Results}
  \label{tab:complete-auroc}
   \resizebox{\linewidth}{!}{%
    \begin{tabular}{l c c c c c c c c c}

\midrule

\textbf{Datasets} & \textbf{Standard} & \textbf{DDU} & \textbf{NC} & \textbf{OE} & \textbf{PreLoad} & \textbf{OE-FT}  &\textbf{DOE-FT} & \textbf{Energy-FT} & \textbf{PreLoad-FT} \\

\midrule

\textbf{MNIST} & & & & & & & & \\
F-MNIST & 98.3$\pm$0.1 & 96.9$\pm$0.1 & 100.0$\pm$0.0 & 99.8$\pm$0.0 & 100.0$\pm$0.0 & 99.7$\pm$0.0 & 99.9$\pm$0.0 & 99.2$\pm$0.2 & 100.0$\pm$0.0 \\
E-MNIST & 90.2$\pm$0.1 & 91.6$\pm$0.1 & 96.4$\pm$0.1 & 92.8$\pm$0.0 & 89.0$\pm$0.2 & 93.7$\pm$0.1 & 94.5$\pm$0.1 & 90.2$\pm$0.1 & 87.7$\pm$1.1 \\
K-MNIST & 97.5$\pm$0.1 & 98.9$\pm$0.0 & 99.3$\pm$0.1 & 98.6$\pm$0.1 & 97.5$\pm$0.2 & 99.0$\pm$0.1 & 99.1$\pm$0.0 & 97.7$\pm$0.1 & 95.2$\pm$1.1 \\
CIFAR-Gr & 99.8$\pm$0.0 & 94.3$\pm$0.4 & 100.0$\pm$0.0 & 100.0$\pm$0.0 & 100.0$\pm$0.0 & 100.0$\pm$0.0 & 100.0$\pm$0.0  & 100.0$\pm$0.0 & 100.0$\pm$0.0 \\
Uniform & 96.7$\pm$0.5 & 93.2$\pm$0.8 & 100.0$\pm$0.0 & 100.0$\pm$0.0 & 100.0$\pm$0.0 & 100.0$\pm$0.0 & 100.0$\pm$0.0 & 99.6$\pm$0.1 & 100.0$\pm$0.0 \\
Smooth & 100.0$\pm$0.0 & 89.4$\pm$2.7 & 100.0$\pm$0.0 & 100.0$\pm$0.0 & 100.0$\pm$0.0 & 100.0$\pm$0.0 & 100.0$\pm$0.0 & 100.0$\pm$0.0 & 100.0$\pm$0.0 \\
FarAway & 1.1$\pm$0.1 & 100.0$\pm$0.0 & 100.0$\pm$0.0 & 59.8$\pm$16.0 & 100.0$\pm$0.0 & 7.3$\pm$4.6 & 43.9$\pm$18.1 & 0.0$\pm$0.0 & 100.0$\pm$0.0 \\
FarAway-RD & 1.2$\pm$0.1 & 100.0$\pm$0.0 & 16.4$\pm$1.0 & 1.5$\pm$0.1 & 100.0$\pm$0.0 & 2.1$\pm$0.3  & 0.6$\pm$0.1 & 0.0$\pm$0.0 & 100.0$\pm$0.0 \\

\midrule

\textbf{F-MNIST} & & & & & & & & \\
MNIST & 80.1$\pm$0.6 & 99.7$\pm$0.1 & 98.4$\pm$0.1 & 84.4$\pm$0.7 & 98.4$\pm$0.5 & 82.9$\pm$0.4 & 88.4$\pm$0.4 & 90.5$\pm$1.0 & 92.4$\pm$1.2 \\
E-MNIST & 82.6$\pm$0.4 & 99.3$\pm$0.1 & 99.6$\pm$0.1 & 88.8$\pm$0.7 & 99.7$\pm$0.1 & 87.5$\pm$0.6  & 93.8$\pm$0.7 & 93.5$\pm$0.4 & 96.1$\pm$1.0 \\
K-MNIST & 83.1$\pm$0.4 & 99.7$\pm$0.0 & 99.0$\pm$0.1 & 89.9$\pm$0.2 & 98.9$\pm$0.2 & 89.1$\pm$0.1  & 93.8$\pm$0.2 & 95.7$\pm$0.3 & 95.1$\pm$0.6 \\
CIFAR-Gr & 83.8$\pm$0.6 & 85.0$\pm$1.2 & 100.0$\pm$0.0 & 100.0$\pm$0.0 & 100.0$\pm$0.0 & 100.0$\pm$0.0 & 100.0$\pm$0.0 & 99.8$\pm$0.0 & 100.0$\pm$0.0 \\
Uniform & 74.3$\pm$3.3 & 95.7$\pm$1.2 & 99.9$\pm$0.1 & 98.6$\pm$1.1 & 100.0$\pm$0.0 & 99.9$\pm$0.0  & 99.7$\pm$0.2 & 99.8$\pm$0.1 & 99.6$\pm$0.3 \\
Smooth & 95.9$\pm$0.2 & 59.5$\pm$1.2 & 100.0$\pm$0.0 & 100.0$\pm$0.0 & 100.0$\pm$0.0 & 99.9$\pm$0.0  & 100.0$\pm$0.0  & 99.1$\pm$0.0 & 98.6$\pm$1.3 \\
FarAway & 2.6$\pm$0.2 & 100.0$\pm$0.0 & 49.0$\pm$21.5 & 2.2$\pm$0.3 & 100.0$\pm$0.0 & 1.8$\pm$0.2  & 1.8$\pm$1.5 & 61.6$\pm$8.9 & 100.0$\pm$0.0 \\
FarAway-RD & 2.8$\pm$0.2 & 100.0$\pm$0.0 & 5.3$\pm$0.5 & 2.3$\pm$0.2 & 100.0$\pm$0.0 & 1.7$\pm$0.2  & 0.4$\pm$0.1 & 18.5$\pm$8.8 & 100.0$\pm$0.0 \\

\midrule

\textbf{SVHN} & & & & & & & & \\
CIFAR-10 & 95.9$\pm$0.1 & 98.3$\pm$0.1 & 100.0$\pm$0.0 & 100.0$\pm$0.0 & 100.0$\pm$0.0 & 99.9$\pm$0.0 & 100.0$\pm$0.0 & 98.8$\pm$0.0 & 100.0$\pm$0.0 \\
LSUN-CR & 95.6$\pm$0.1 & 99.1$\pm$0.0 & 100.0$\pm$0.0 & 100.0$\pm$0.0 & 100.0$\pm$0.0 & 100.0$\pm$0.0 & 100.0$\pm$0.0 & 99.1$\pm$0.1 & 100.0$\pm$0.0 \\
CIFAR-100 & 95.1$\pm$0.1 & 98.2$\pm$0.1 & 100.0$\pm$0.0 & 100.0$\pm$0.0 & 100.0$\pm$0.0 & 99.9$\pm$0.0 & 100.0$\pm$0.0 & 98.2$\pm$0.0 & 100.0$\pm$0.0 \\
FMNIST-3D & 94.5$\pm$0.4 & 96.0$\pm$0.3 & 100.0$\pm$0.0 & 100.0$\pm$0.0 & 100.0$\pm$0.0 & 99.7$\pm$0.1 & 100.0$\pm$0.0 & 96.9$\pm$0.3 & 100.0$\pm$0.0 \\
Uniform & 93.2$\pm$0.9 & 100.0$\pm$0.0 & 100.0$\pm$0.0 & 100.0$\pm$0.0 & 100.0$\pm$0.0 & 100.0$\pm$0.0 & 100.0$\pm$0.0 & 99.6$\pm$0.0 & 100.0$\pm$0.0 \\
Smooth & 96.9$\pm$0.2 & 99.2$\pm$0.1 & 97.1$\pm$0.3 & 98.0$\pm$0.1 & 98.4$\pm$0.2 & 98.3$\pm$0.1 & 98.5$\pm$0.1 & 97.4$\pm$0.2 & 98.9$\pm$0.0 \\
FarAway & 1.8$\pm$0.6 & 100.0$\pm$0.0 & 20.2$\pm$20.0 & 2.0$\pm$1.2 & 100.0$\pm$0.0 & 3.9$\pm$1.1 & 1.1$\pm$0.2 & 0.0$\pm$0.0 & 100.0$\pm$0.0 \\
FarAway-RD & 19.9$\pm$5.9 & 100.0$\pm$0.0 & 20.2$\pm$20.0 & 38.0$\pm$16.5 & 100.0$\pm$0.0 & 23.0$\pm$7.2 & 3.3$\pm$2.3 & 0.0$\pm$0.0 & 100.0$\pm$0.0 \\

\midrule

\textbf{CIFAR-10} & & & & & & & & \\
SVHN & 94.4$\pm$1.0 & 91.1$\pm$0.5 & 99.6$\pm$0.1 & 99.4$\pm$0.1 & 99.7$\pm$0.1 & 98.6$\pm$0.3 & 99.1$\pm$0.1 & 98.9$\pm$0.1 & 99.1$\pm$0.3 \\
LSUN-CR & 92.6$\pm$0.2 & 87.4$\pm$0.3 & 99.7$\pm$0.0 & 99.0$\pm$0.0 & 99.7$\pm$0.0 & 97.1$\pm$0.1  & 97.9$\pm$0.1 & 98.4$\pm$0.0 & 99.1$\pm$0.0 \\
CIFAR-100 & 90.0$\pm$0.0 & 80.7$\pm$0.2 & 94.9$\pm$0.0 & 94.1$\pm$0.1 & 94.3$\pm$0.1 & 92.3$\pm$0.1 & 93.0$\pm$0.0 & 93.4$\pm$0.0 & 92.9$\pm$0.0 \\
FMNIST-3D & 94.7$\pm$0.2 & 94.5$\pm$0.8 & 99.3$\pm$0.1 & 98.3$\pm$0.1 & 99.1$\pm$0.1 & 97.7$\pm$0.1 & 97.9$\pm$0.2 & 98.4$\pm$0.1 & 98.5$\pm$0.2 \\
Uniform & 88.2$\pm$3.0 & 100.0$\pm$0.0 & 99.8$\pm$0.2 & 100.0$\pm$0.0 & 99.8$\pm$0.0 & 100.0$\pm$0.0  & 100.0$\pm$0.0 & 97.4$\pm$1.8 & 100.0$\pm$0.0 \\
Smooth & 93.6$\pm$1.1 & 96.1$\pm$0.7 & 99.0$\pm$0.1 & 97.5$\pm$0.4 & 98.9$\pm$0.1 & 96.3$\pm$0.4  & 96.5$\pm$0.4 & 96.9$\pm$0.4 & 97.1$\pm$0.7 \\
FarAway & 3.4$\pm$0.1 & 100.0$\pm$0.0 & 80.1$\pm$19.9 & 3.6$\pm$0.0 & 100.0$\pm$0.0 & 5.1$\pm$0.3 & 1.5$\pm$0.4 & 0.0$\pm$0.0 & 100.0$\pm$0.0 \\
FarAway-RD & 5.0$\pm$1.3 & 100.0$\pm$0.0 & 60.3$\pm$24.3 & 3.6$\pm$0.0 & 100.0$\pm$0.0 & 10.1$\pm$3.6 & 34.2$\pm$16.6 & 0.0$\pm$0.0 & 100.0$\pm$0.0 \\

\midrule

\textbf{CIFAR-100} & & & & & & & & \\
SVHN & 82.9$\pm$0.5 & 72.0$\pm$2.4 & 96.1$\pm$0.9 & 86.7$\pm$0.7 & 96.6$\pm$0.3 & 84.0$\pm$0.9 & 89.6$\pm$1.2 & 91.7$\pm$1.0 & 92.1$\pm$1.0 \\
LSUN-CR & 80.1$\pm$0.5 & 73.5$\pm$0.6 & 97.4$\pm$0.1 & 86.6$\pm$0.4 & 96.6$\pm$0.0 & 83.0$\pm$0.4  & 93.1$\pm$0.3 & 92.0$\pm$0.2 & 93.6$\pm$0.1 \\
CIFAR-10 & 77.4$\pm$0.2 & 65.7$\pm$0.3 & 80.8$\pm$0.1 & 77.2$\pm$0.2 & 79.1$\pm$0.5 & 77.1$\pm$0.2  & 74.6$\pm$0.3 & 79.3$\pm$0.1 & 73.9$\pm$0.2 \\
FMNIST-3D & 85.2$\pm$0.9 & 67.6$\pm$1.1 & 98.3$\pm$0.1 & 88.4$\pm$0.7 & 97.2$\pm$0.2 & 86.5$\pm$0.6 & 94.7$\pm$0.2 & 94.7$\pm$0.5 & 92.8$\pm$0.4 \\
Uniform & 62.3$\pm$7.7 & 99.9$\pm$0.1 & 100.0$\pm$0.0 & 87.1$\pm$8.5 & 100.0$\pm$0.0 & 73.8$\pm$10.2 & 94.8$\pm$1.6 & 95.9$\pm$2.9 & 100.0$\pm$0.0 \\
Smooth & 75.4$\pm$1.2 & 84.9$\pm$2.2 & 89.9$\pm$1.2 & 81.1$\pm$1.3 & 95.1$\pm$1.0 & 78.3$\pm$0.7 & 79.9$\pm$1.7 & 82.7$\pm$1.5 & 96.1$\pm$0.8 \\
FarAway & 0.4$\pm$0.0 & 100.0$\pm$0.0 & 80.2$\pm$19.8 & 0.8$\pm$0.0 & 100.0$\pm$0.0 & 0.9$\pm$0.1 & 0.1$\pm$0.1 & 0.0$\pm$0.0 & 100.0$\pm$0.0 \\
FarAway-RD & 1.1$\pm$0.3 & 100.0$\pm$0.0 & 80.2$\pm$19.8 & 0.9$\pm$0.2 & 100.0$\pm$0.0 & 1.0$\pm$0.1  & 1.5$\pm$0.5 & 0.0$\pm$0.0 & 100.0$\pm$0.0 \\

\midrule

\end{tabular}

    }
\end{table}

\begin{table}
\caption{Accuracy}
\label{tab:accuracy}
\centering
\begin{tabular}{l c c c c c }

\midrule

& \textbf{MNIST} & \textbf{F-MNIST} & \textbf{SVHN} & \textbf{CIFAR-10} & \textbf{CIFAR-100} \\

\midrule

Standard & 99.5 & 92.7 & 97.4 & 94.9 & 77.3 \\

\midrule
DDU & 99.5 & 92.7 & 97.4 & 94.9 & 77.3 \\

\midrule
NC & 99.4 & 92.6 & 97.3 & 92.8 & 73.6 \\

\midrule
OE & 99.5 & 92.7 & 97.4 & 95.5 & 77.1 \\

\midrule
PreLoad & 99.5 & 92.3 & 97.3 & 93.5 & 71.9 \\

\midrule
OE-FT & 99.5 & 92.8 & 97.3 & 94.8 & 77.1 \\

\midrule
DOE-FT & 99.5 & 92.8 & 97.3 & 94.6 & 74.1 \\

\midrule
Energy-FT & 99.5 & 92.8 & 97.4 & 94.9 & 76.8 \\

\midrule
PreLoad-FT & 99.5 & 92.4 & 97.1 & 94.5 & 76.8 \\

\midrule
\end{tabular}
\end{table}

\begin{table}
\centering
\caption{Calibration measured using the ECE score}
\label{tab:calibration}
\begin{tabular}{l c c c c c }

\midrule

& \textbf{MNIST} & \textbf{F-MNIST} & \textbf{SVHN} & \textbf{CIFAR-10} & \textbf{CIFAR-100} \\

\midrule

Standard & 7.1 & 12.2 & 8.9 & 10.6 & 13.3 \\

\midrule
DDU & 7.1 & 12.2 & 8.9 & 10.6 & 13.3 \\

\midrule
NC & 7.1 & 7.4 & 9.0 & 15.2 & 8.6 \\

\midrule
OE & 6.4 & 11.5 & 9.1 & 6.1 & 13.1 \\

\midrule
PreLoad & 8.9 & 6.1 & 7.3 & 11.9 & 9.8 \\

\midrule
OE-FT & 7.1 & 11.6 & 10.0 & 5.7 & 16.0 \\

\midrule

DOE-FT & 7.3 & 7.5 & 8.2 & 15.7 & 16.9 \\

\midrule
Energy-FT & 6.7 & 11.8 & 12.2 & 10.3 & 15.5 \\

\midrule
PreLoad-FT & 8.6 & 3.9 & 10.3 & 8.1 & 11.1 \\

\midrule
\end{tabular}
\end{table}

\end{appendices}

\end{document}